\documentclass[lettersize,journal]{IEEEtran}

\usepackage{cite}

\usepackage{hyperref}       % hyperlinks
\usepackage{url}            % simple URL typesetting
\usepackage{booktabs}       % professional-quality tables
\usepackage{amsfonts}       % blackboard math symbols
\usepackage{nicefrac}       % compact symbols for 1/2, etc.
\usepackage{microtype}      % microtypography
\usepackage[table]{xcolor}
\definecolor{mycolor}{gray}{0.92}

\usepackage{microtype}
\usepackage{graphicx}
\usepackage{subcaption}
\usepackage{booktabs} % for professional tables
\usepackage{url}
\usepackage{wrapfig}
\usepackage[normalem]{ulem}

\usepackage{hyperref}

\usepackage{amsmath}
\usepackage{amssymb}
\usepackage{mathtools}
\usepackage{amsthm}
\usepackage{multirow}
\usepackage{makecell}

\usepackage{algorithm}
\usepackage{algorithmic}

\usepackage[capitalize,noabbrev]{cleveref}
\usepackage{thmtools}

\theoremstyle{plain}
\newtheorem{theorem}{Theorem}[section]
\newtheorem{proposition}[theorem]{Proposition}
\newtheorem{lemma}[theorem]{Lemma}

\theoremstyle{definition}

\newtheorem{assumption}[theorem]{Assumption}
\theoremstyle{remark}

\usepackage{placeins}

\begin{document}

\title{Preserving Domain Generalization in Fine-Tuning via Joint Parameter Selection}

\author{Bin Pan, Shiyu Shen, Zongbin Wang, Zhenwei Shi and Xia Xu
    
    \thanks{The work was supported by the National Key Research and Development Program of China under Grant 2022YFA1003800, and the Fundamental Research Funds for the Central Universities under grant 63243074. \emph{(Corresponding author: Xia Xu.)}}
    
    \thanks{Bin Pan, Shiyu Shen and Zongbin Wang are with the School of Statistics and Data Science, KLMDASR, LEBPS, and LPMC, Nankai University, Tianjin 300071, China
        (e-mail: panbin@nankai.edu.cn; shenshiyu@mail.nankai.edu.cn; wangzongbin@mail.nankai.edu.cn).}
    \thanks{Zhenwei Shi are with the Image Processing Center,
        School of Astronautics, Beihang University, Beijing 100191, China (e-mail: shizhenwei@buaa.edu.cn).}
    \thanks{Xia Xu (corresponding author) is with the School of Computer Science and Technology, Tiangong University, Tianjin 300387, China (e-mail: xuxia@tiangong.edu.cn).}
}

\markboth{under review}%
{Shell \MakeLowercase{\textit{et al.}}: A Sample Article Using IEEEtran.cls for IEEE Journals}

\maketitle

\begin{abstract}
Domain generalization seeks to develop models trained on a limited set of source domains that are capable of generalizing effectively to unseen target domains. While the predominant approach leverages large-scale pre-trained vision models as initialization, recent studies have highlighted that full fine-tuning can compromise the intrinsic generalization capabilities of these models. To address this limitation, parameter-efficient adaptation strategies have emerged, wherein only a subset of model parameters is selectively fine-tuned, thereby balancing task adaptation with the preservation of generalization. Motivated by this paradigm, we introduce Joint Parameter Selection (JPS), a novel method that restricts updates to a small, sparse subset of parameters, thereby retaining and harnessing the generalization strength of pre-trained models. Theoretically, we establish a generalization error bound that explicitly accounts for the sparsity of parameter updates, thereby providing a principled justification for selective fine-tuning. Practically, we design a selection mechanism employing dual operators to identify and update parameters exhibiting consistent and significant gradients across all source domains. Extensive benchmark experiments demonstrate that JPS achieves superior performance compared to state-of-the-art domain generalization methods, substantiating both the efficiency and efficacy of the proposed approach.
\end{abstract}


\section{Introduction\label{submission}}

\begin{figure}[htbp]
	\begin{center}
		\centerline{\includegraphics[width=\columnwidth]{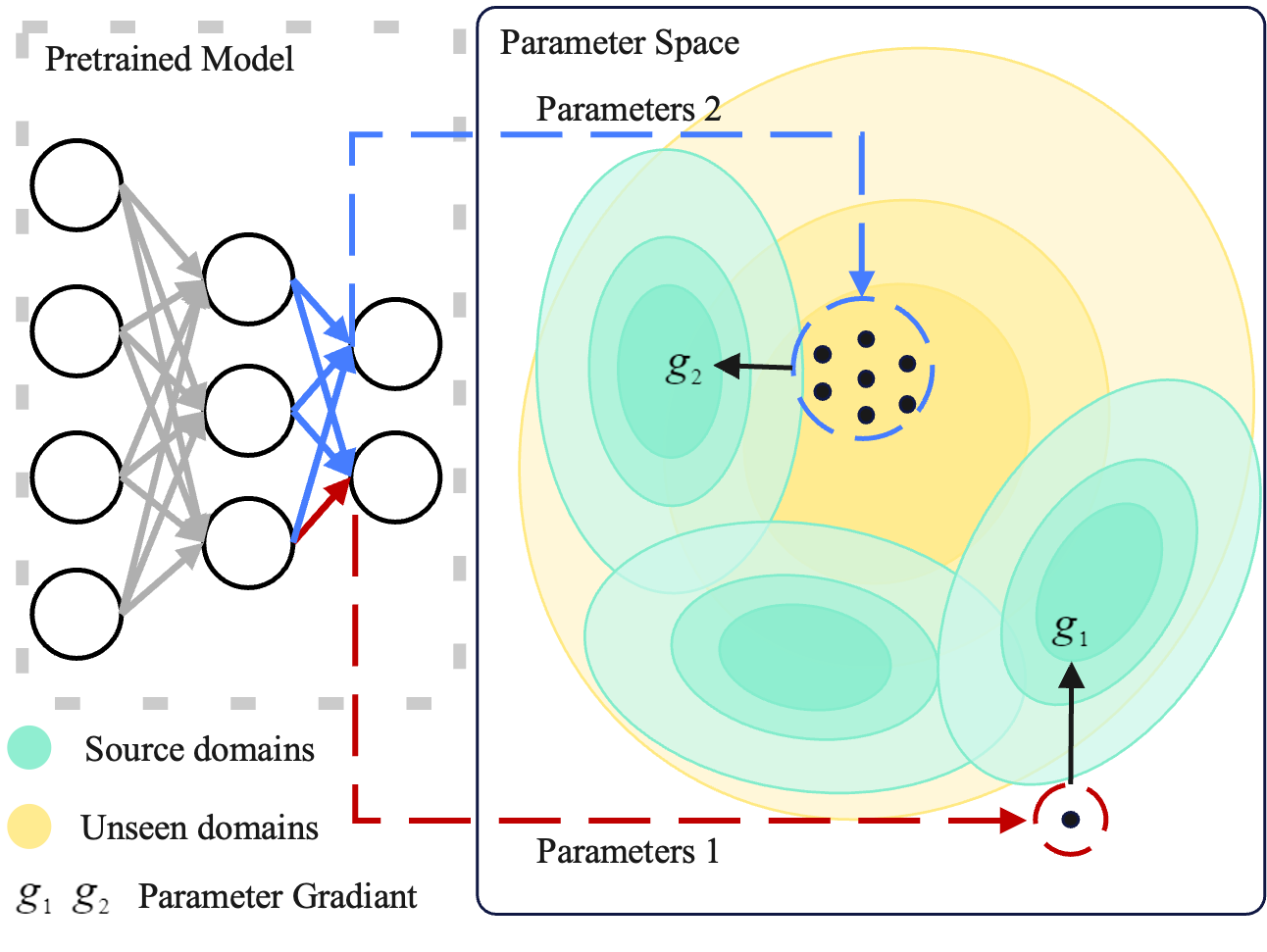}}
		\caption{Illustration of risks in full fine-tuning. When using a pre-trained model, some parameters might already be close to the unseen target domain. Optimizing all parameters might improve performance on source domains but harm generalization.}
		\vspace{-5ex}
		\label{il1}
	\end{center}
\end{figure}

Domain generalization addresses the fundamental challenge posed by distributional discrepancies between training and testing data \cite{9782500}. While the assumption of independent and identically distributed (i.i.d.) data underpins much of machine learning, it is frequently violated in practical applications, resulting in notable performance degradation when models encounter previously unseen domain shifts \cite{zongshu}. In computer vision, such domain shifts may arise from variations in style, texture, background, or camera viewpoint \cite{PACS, VLCS, Office}.

Domain generalization approaches commonly employ fine-tuning strategies on pre-trained backbone models such as CLIP pre-trained ViT\cite{clip}, leveraging the robust initial representations provided by large-scale pre-training \cite{zongshu2,10510661}. These methods typically adapt pre-trained vision models by optimizing them with respect to task-specific feature constraints or specialized losses \cite{10239476,10251546}, with the objective of preserving the generalization capability of the pre-trained models to the downstream task.

Nevertheless, full fine-tuning can undermine the general-purpose knowledge embedded in pre-trained vision models \cite{SPU, PEGO, lora}, potentially altering well-optimized parameters to enhance source domain performance at the expense of generalization. As illustrated in \cref{il1}, certain pre-trained parameters may already be optimally aligned with unseen domains, and indiscriminate updates based on source domain data can inadvertently degrade their utility. Furthermore, the need to retrain fully fine-tuned models for related, yet distinct, tasks imposes considerable computational and communication overhead, which is particularly prohibitive in resource-constrained environments such as satellite platforms or autonomous vehicles \cite{weixing1, zidong1}.

Consequently, selective fine-tuning —wherein only a sparse subset of parameters is updated— has emerged as a promising alternative, balancing adaptation and generalization while minimizing parameter modifications \cite{PM}. Recent advances indicate that sparse parameter optimization not only enhances stability but also yields superior generalization compared to full model updates \cite{SPU, Fed, onthe}. However, the core challenge lies in carefully selecting which parameters to update. Existing methods often rely on gradient magnitude as a criterion, presuming that parameters with the largest gradients are most relevant to the current task. This assumption is questionable in the context of domain generalization, where parameters significant in one domain may be uninformative or even detrimental in others, thereby introducing the risk of overfitting \cite{Fish, Fishr}. This underscores the necessity for parameter selection strategies that explicitly account for domain-specific gradient behavior.

To address these limitations, we propose Joint Parameter Selection (JPS), a novel framework for domain generalization that leverages gradient information across multiple domains to identify parameters with robust generalization potential. Our approach is grounded in rigorous theory, introducing a generalization error upper bound that incorporates update sparsity, thereby elucidating the interaction between sparsity and model generalization. Motivated by this analysis, we propose a two-stage selection mechanism: we choose sparse and small sets of ViT parameters with consistent, significant gradients across source domains in specific layers for updating. 

This adaptive framework affords flexible control over the computational budget by modulating both the degree of sparsity and the choice of layers for fine-tuning, yielding two principal operational regimes:
\begin{enumerate}
	\item Enhanced Generalization with Comparable Cost. By selecting enough layers, we improve generalization while keeping the computational cost similar to full fine-tuning.
	\item Reduced Resource Consumption with Competitive Generalization. By selecting limited layers, we reduce training resources while achieving competitive performance.
\end{enumerate}

We empirically validate our method on the DomainBed benchmark \cite{domainbed}, demonstrating that JPS consistently outperforms state-of-the-art domain generalization and parameter-efficient adaptation methods. In summary, the key contributions of this work are as follows:
\begin{itemize}
	\item We introduce a selective fine-tuning method that updates only a small set of parameters to preserve and leverage the generalization ability of the pre-trained model for domain generalization.
	\item We propose a new generalization upper bound and show that a suitable parameter selection strategy reduces prediction error on unseen domains.
	\item We present a generic implementation with two operators to select and update the parameters with strong generalization potential.
\end{itemize}

\section{Related Work and Background}

\subsection{Large-Scale Pre-trained Vision Models}

Large-scale pre-trained vision models have become foundational in computer vision, offering powerful and transferable feature representations. Early approaches relied on supervised pre-training using large-scale datasets such as ImageNet, with models like ResNet \cite{resnet} and Vision Transformer (ViT) \cite{vit} achieving impressive performance across various tasks. In recent years, self-supervised pre-training methods \cite{chen2020simclr, he2020moco} have emerged, enabling models to learn rich and generalizable features without requiring labeled data. Among these, CLIP \cite{clip} introduced a paradigm shift by jointly training on image-text pairs, resulting in models that understand both visual and semantic concepts. Notably, CLIP-pretrained ViT models have become some of the most widely used and effective backbones for downstream vision tasks due to their strong generalization capabilities and versatility.

\subsection{Domain Generalization}
Domain generalization aims to use multiple source domains \( S = \{S_1, S_2, \dots, S_N\} \) to train a model that performs robustly on unseen target domains \( T, T \cap S = \emptyset \). Each \( S_i \) is a distribution from which data \( (x, y) \) are drawn: \( (x, y) \sim S_i \), where \( x \) and \( y \) are the samples and labels. The objective is to learn a model within the hypothesis space \( h \in \mathcal{H} \) and parameter $\theta \in \Theta$, guided by a loss function \( l(h_{\theta}(x), y) \). The ideal loss function in domain generalization is \cite{irm}:
\begin{equation}
R(h_{\theta},D)=E_{(x,y)\sim D}(l(h_{\theta}(x),y))\ ,D=S\cup T.
\end{equation}
In practice, $T$ and the distribution of $S$ are inaccessible during training \cite{zongshu}. Only the training data from source domains $\hat{S} = \{x_i, y_i\}_{i=1}^n,\hat{S}\subset S$ is available. The empirical loss function in domain generalization is:
\begin{equation}
	R(h_{\theta},\hat{S}) = \frac{1}{n}\sum_{i=1}^n l(h_{\theta}(x_i),y_i)\ ,(x_i, y_i)\in \hat{S}. \label{emp}
\end{equation}

This loss in domain generalization, known as Empirical Risk Minimization (ERM) \cite{ermi}, minimizes task loss on source domains. While ERM is a strong baseline \cite{domainbed}, it does not leverage the diversity within $S$. Current methods address this using different hypotheses or techniques \cite{zongshu, zongshu2}. Some assume invariant features across domains, extractable through feature alignment \cite{irm, CORAL, DANN}, while others use data augmentation, enhancing robustness and generalization \cite{mix, puzzle}. Gradient-based methods like \cite{Fish, Fishr} align gradients across domains, while \cite{SAGM} regulates gradients for smoother solutions. Many also utilize pre-trained models \cite{clip, softprompt} for domain generalization, such as \cite{miro} extracting features, \cite{GES} estimating gradients for unseen domains, and \cite{PEGO} applying LoRA \cite{lora}. Our method also leverages pre-trained knowledge but focuses on selecting parameters to improve generalization.

\begin{figure*}[htbp]
	\centering
	\includegraphics[width=1\textwidth]{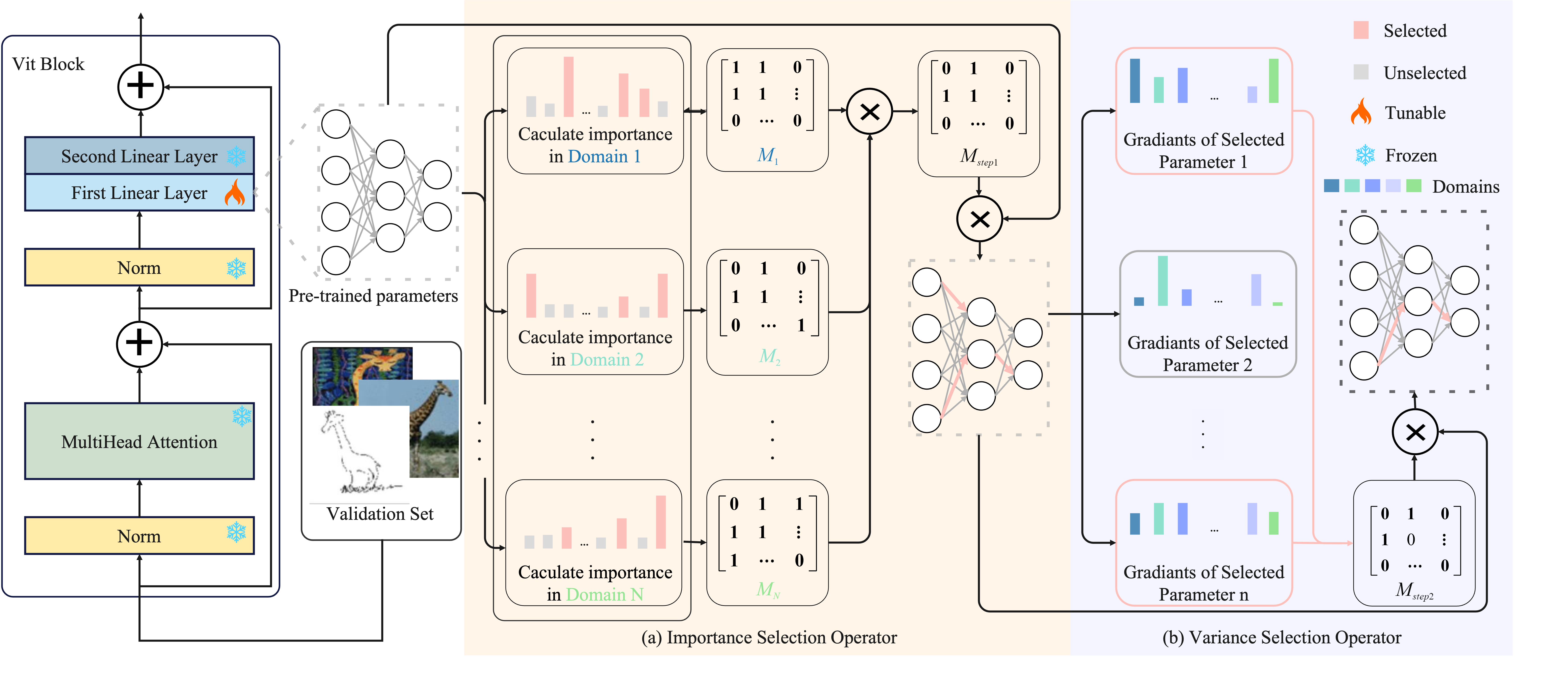}
	\caption{Flowchart of the proposed implementation. We choose the first layer of the ViT module for selecting and updating. The $M^{step2}$ is generated before the update, using gradients from the pre-trained model. In the Importance Selection Operator, we use $M_{step1}$ to select parameters with significant gradients across all domains. In the Variance Selection Operator, we calculate the gradient variance of the remaining parameters across domains and select those with lower variance, resulting in the final mask $M^{step2}$. Once $M^{step2}$ is generated, it remains unchanged.}
	\label{alg}
\end{figure*}

\subsection{Sparse Parameter Updates}
Assume the pre-trained model is $h \in \mathcal{H}$ with parameters $\theta_0 \in \Theta$, where $\theta_0$ has $m$ dimension, we can set the sparse optimization objective as follows \cite{onthe}:
\begin{gather}
	\min_{\nabla\theta,M}R(h_{\theta_0+M\nabla\theta},\hat{S})\nonumber,\\
	s.t. ||M||_0 = [m\rho], M_{ij} = 0 ,M_{ii}\in \{0,1\}. \label{spas}
\end{gather}
Here, $M$ is the mask of size $m \times m$, where $M_{ii}$ indicates whether the $i$-th parameter needs updating. The parameter $\rho$ represents sparsity, showing the proportion of parameters that require updates. The rounding function is denoted by $[\cdot]$. To determine the values of  $M_{ii}$, various methods focus on selecting parameters for updates \cite{nature1}. Some propose fine-tuning specific layers using trainable adapters \cite{lora, adapter, PEGO}, while others rely on gradients to identify significant parameters  \cite{onthe, Fed, SPU, PU}. Our method integrates these strategies, selecting parameters from specific layers while refining the process for different distributions in the source domains.

\section{Joint Parameter Selection\label{JPS}}

The proposed method includes a theoretical analysis and an implementation process. \cref{JPS1} introduces the upper bound using sparse updates to relate training and test set prediction losses, along with corresponding propositions to guide our implementation. Building on this analysis, \cref{JPS2} focuses on the practical implementation of the JPS algorithm, starting with an overview of the process and then a detailed examination of the implementation specifics.

\subsection{Bounding Generalization Error with Sparsity Constraints}
\label{JPS1}
We define the sparse learning method \( A \), where \( A(\theta, S) \) represents the set of model parameters \( \theta = \theta_0 + M\nabla\theta \), learned from the data \( S \) to minimize certain optimization goals. For convenience, we define the loss function of the model learned by $A$ as:
$R(A(\theta, S), S) = E_{(x,y) \sim S} l(h_{A(\theta, S)}(x), y)$. We investigate the relationship between $R(A(\theta,\hat{S}), \hat{S})$ and $R(A(\theta,\hat{S}), T)$. We first present a lemma and Theorem to illustrate the relationship between sparse updates and model loss. We then derive two propositions to show how parameter selection can reduce prediction loss.

\begin{lemma}\label{lemma3.1}
Assume that the loss function \( R \) exhibits \( p \)-Lipschitz continuity and Pointwise Hypothesis Stability \( c \) \cite{stable} for method \( A \). Set \( |\nabla_{\theta} R(A(\theta, S), S)| \geq C_{min}^{S} \cdot I \). Let \( \beta = \max(p, c) \), we have:
\begin{align}
		E_{\hat{S},i\sim U(n)}&(|l(h_{A(\theta,\hat{S})}(x_i),y_i)-l(h_{A(\theta,\hat{S}^{/i})}(x_i),y_i)|)\leq \mathcal{K}_1,\nonumber\\
		&\mathcal{K}_1 = \frac{\beta^2}{(C_{min}^{\hat{S}}+2(1-\rho))n}, (x_i,y_i)\in \hat{S}.
\end{align}
\end{lemma}

The proof are shown in Appendix \cref{lemma3.1_proof}. Here, \( U(\cdot) \) represents the uniform distribution. The sparse updates compress the pointwise stability from $c$ to $\frac{\beta^2}{(C_{min}^{S}+2(1-\rho))n}$, demonstrating that sparse updates can enhance the stability. Next, using \cref{lemma3.1} and the $\mathcal{H}$-divergence \cite{hsandu}, we derive the loss relationship between the $R(A(\theta,\hat{S}), \hat{S})$ and $R(A(\theta,\hat{S}), T)$: 

\begin{theorem}\label{theorem3_2}
Assume $|\nabla_{\theta} R(A(\theta, S), S_i)| \geq C_{\min}^{S_i}$. Under the same assumption as \cref{lemma3.1}, we have the following inequality with probability $1 - \sigma$ for some constant $C$:
\begin{multline}\label{th1}
	R(A(\theta,\hat{S}),T)\leq R(A(\theta,\hat{S}),\hat{S})\\
	+\sqrt{\frac{C^2+12Cn\mathcal{K}_2}{2n\sigma}}+\frac{1}{2}\mathcal{H}\nabla \mathcal{H}(S,T),
\end{multline}
\begin{equation}
		\mathcal{K}_2 \!= (\frac{1}{N}\!\sum_{i=1}^N\frac{\beta^2}{(C_{min}^{S^i}+2(1-\rho))n})+\max\limits_{i}\frac{1}{2}\mathcal{H}\nabla \mathcal{H}(\hat{S}_i,\hat{S}).	\nonumber
\end{equation}
\end{theorem}
In this context, $n$ represents the data points in $\hat{S}$, and $\rho$ refers to the model sparsity. $\mathcal{H}\nabla \mathcal{H}(\cdot,\cdot)$ denotes the $\mathcal{H}$-divergence, a widely used measure in domain generalization used to quantify the prediction discrepancy between two distributions. The better the generalization ability of the model, the smaller the value of $\mathcal{H}\nabla \mathcal{H}(\cdot,\cdot)$.
\begin{proof}
\label{pf32}
First, we need to introduce two Theorems that will be used in the proof:
\begin{theorem}
	\label{th2}
	We set $\mathcal{H}\nabla \mathcal{H}$ hypothesis space as :
	$g(x)\in \mathcal{H}\nabla \mathcal{H} \Rightarrow g(x)=h(x)\oplus h'(x)$, for some $h,h'\in \mathcal{H}$,$\oplus$ means the XOR operation.
	And we can set the $d_{\mathcal{H}\nabla \mathcal{H}}(S,T)$ as:
	\begin{multline}
		d_{\mathcal{H}\nabla \mathcal{H}}(S,T)=\\
		2\sup_{h,h'\in \mathcal{H}}|P_{x\sim S}(h(x)\neq h'(x))-P_{x\sim T}(h(x)\neq h'(x))|
	\end{multline}
	then the inequality holds: 
	\begin{equation}
		R_T(h)\leq R_S(h)+\frac{1}{2}d_{\mathcal{H}\nabla \mathcal{H}}(S,T)+\lambda_{\mathcal{H}}. \label{dh} 
	\end{equation}
\end{theorem}

Here, $\lambda$ is related to $S$, $T$, and $\mathcal{H}$, but independent of the specific $h$. $d_{\mathcal{H}\nabla \mathcal{H}}(S,T)$ represents the comparison between the current function $h$ and the function $h^*_{S,T}$ that minimizes the loss $R$ on both $S$ and $T$. For simplicity, we simplify $\frac{1}{2}d_{\mathcal{H}\nabla \mathcal{H}}(S,T) + \lambda_{\mathcal{H}}$ as $\frac{1}{2}\mathcal{H}\nabla\mathcal{H}(S,T)$ in the main text.

\begin{theorem}
	\label{th3}
	(Theorem 11 in \cite{stable}) For any learning algorithm $A$ with pointwise hypothesis stability $\beta$ with respect to a loss function such that $0\leq c(y,y')\leq C$, we have with probability $1-\sigma$:
	\begin{equation}
		R(A(\theta,S),S)\leq R(A(\theta,\hat{S}),\hat{S})+\sqrt{\frac{C^2+12Cn\beta}{2n\sigma}}.
	\end{equation}
	where, $c(y, y') = |y -y'|$ is an absolute loss function.
\end{theorem}

Next, we will extend \cref{lemma3.1} using domain label information. Suppose $\hat{S} = {\hat{S}_1, \hat{S}_2, ..., \hat{S}_n}$, the following inequality holds:
\begin{multline}
	l(h_{A(\theta,\hat{S})}(x_i),y_i)-l(h_{A(\theta,\hat{S}^{/i})}(x_i),y_i)|
	\\ \leq |l(h_{A(\theta,\hat{S})}(x_i),y_i)-l(h_{A(\theta,\hat{S}_j)}(x_i),y_i)|\\
	+|l(h_{A(\theta,\hat{S}_j)}(x_i),y_i)-l(h_{A(\theta,\hat{S}^{/i})}(x_i),y_i)|. \label{pf7}
\end{multline}
The first term here can be substituted into the \cref{pf6}. At this point, we only consider the samples in $\hat{S}_j$. Therefore, we simply need to replace the maximum value from $C{min}^{\hat{S}}$ to $C_{min}^{\hat{S}_j}$ to obtain a result similar to \cref{lemma3.1}:
\begin{multline}
	E_{\hat{S}_j,i\sim U(n)}|l(h_{A(\theta,\hat{S})}(x_i),y_i)-l(h_{A(\theta,\hat{S}_j)}(x_i),y_i)|\\
	\leq \frac{\beta^2}{n(C_{min}^{\hat{S}_j}+2\lambda)}, (x_i,y_i)\in \hat{S}_j.
\end{multline}
For the second term in \cref{pf7}, we can see that its expectation is a variant of the $\mathcal{H}$-divergence:
\begin{multline}
	E_{\hat{S}_j,i\sim U(n)}(|l(h_{A(\theta,\hat{S}_j)}(x_i),y_i)-l(h_{A(\theta,\hat{S}^{/i})}(x_i),y_i)|)\\
	\leq \frac{1}{2}\mathcal{H}\nabla\mathcal{H}(\hat{S}^{/i},\hat{S}_j)\leq \frac{1}{2}\mathcal{H}\nabla\mathcal{H}(\hat{S},\hat{S}_j).
\end{multline}
 The above inequality holds for any $j$, but the samples may not come from $\hat{S}_j$. Therefore, we need to take the average over $j$, yielding:
\begin{multline}
	E_{\hat{S},i\sim U(n)}|l(h_{A(\theta,\hat{S})}(x_i),y_i)-l(h_{A(\theta,\hat{S}^{/i})}(x_i),y_i)|\\
	\leq E_{\hat{S},i\sim U(n)}I(x_i\in \hat{S}_j)(\frac{1}{N}\sum_{j=1}^N \frac{\beta^2}{n(C_{min}^{\hat{S}_j}+2\lambda)}\\
	+\frac{1}{n}\sum_{j=1}^N \mathcal{H}\nabla\mathcal{H}(\hat{S},\hat{S}_j))\\
	\leq \frac{1}{N}\sum_{j=1}^N \frac{\beta^2}{n(C_{min}^{\hat{S}_j}+2\lambda)}+\max_j \frac{1}{2}\mathcal{H}\nabla\mathcal{H}(\hat{S},\hat{S}_j).
\end{multline}
At this point, let the right-hand side of the above inequality be $\frac{\mathcal{K_2}}{n}$. Using \cref{th2}, we can obtain the following equation with a confidence of $1 - \sigma$ with some fixed $C$:
\begin{equation}
	R(A(\theta,S),S)\leq R(A(\theta,\hat{S}),\hat{S})+\sqrt{\frac{C^2+12Cn\mathcal{K}_2}{2n\sigma}}.
\end{equation}
By applying the $H$ divergence formula once more and substituting $\lambda = 1 - \rho$, we can obtain our \cref{theorem3_2}:
\begin{multline}
	R(A(\theta,\hat{S}),T)\leq R(A(\theta,\hat{S}),\hat{S})\\
	+\sqrt{\frac{C^2+12Cn\mathcal{K}_2}{2n\sigma}}+\frac{1}{2}\mathcal{H}\nabla \mathcal{H}(S,T),
\end{multline}
\begin{equation}
	\mathcal{K}_2 = \frac{1}{N}\sum_{i=1}^N\frac{\beta^2}{(C_{min}^{S^i}+2(1-\rho))n}+\max\limits_{i}\frac{1}{2}\mathcal{H}\nabla \mathcal{H}(\hat{S}_i,\hat{S}). \nonumber
\end{equation}

\end{proof}

\textbf{Motivation.} The inequality shows that the loss in an unknown target domain depends on the sparsity of the training model, the minimum absolute gradient, and the $\mathcal{H}$-divergence. Direct optimization of $\mathcal{H}\nabla \mathcal{H}(S, T)$ is difficult due to the unavailability of the target domain. However, $\mathcal{K}_2$ is optimizable in \cref{th1}. Notably, $\mathcal{K}_1$ can replace $\mathcal{K}_2$ in \cref{th1}, yielding another inequality that reflects the selection strategy used by most existing models \cite{onthe, Fed, SPU}. As $C_{min}^S$ increases and $\rho$ approaches zero, $\mathcal{K}_1$ decreases, suggesting that updating the model by selecting the largest gradients can satisfy this condition.
\begin{proof}
	Notably, Using \cref{th3} directly with \cref{lemma3.1}, we can get an inequality with $\beta = \mathcal{K}_1$. Then using \cref{th2}, we can get the result. The process is similar to the proof procedure of $\mathcal{K}_2$ mentioned above, but more straightforward, so we will not elaborate further. Notably, $\mathcal{K}_1$ treats $S$ as an I.I.D distribution, which fits the previous work assumption. However, the distribution of source domain is not I.I.D.
\end{proof}

However, the formulation of $\mathcal{K}_1$ treats $S$ as an I.I.D dataset, which may not be suitable for domain generalization. $\mathcal{K}_2$ includes terms for the minimum gradients of each domain and the prediction discrepancies between domains. Based on the expression of $\mathcal{K}_2$, we redesign the parameter selection strategy:

\begin{figure}[tbp]
\begin{center}
	\centerline{\includegraphics[width=\columnwidth]{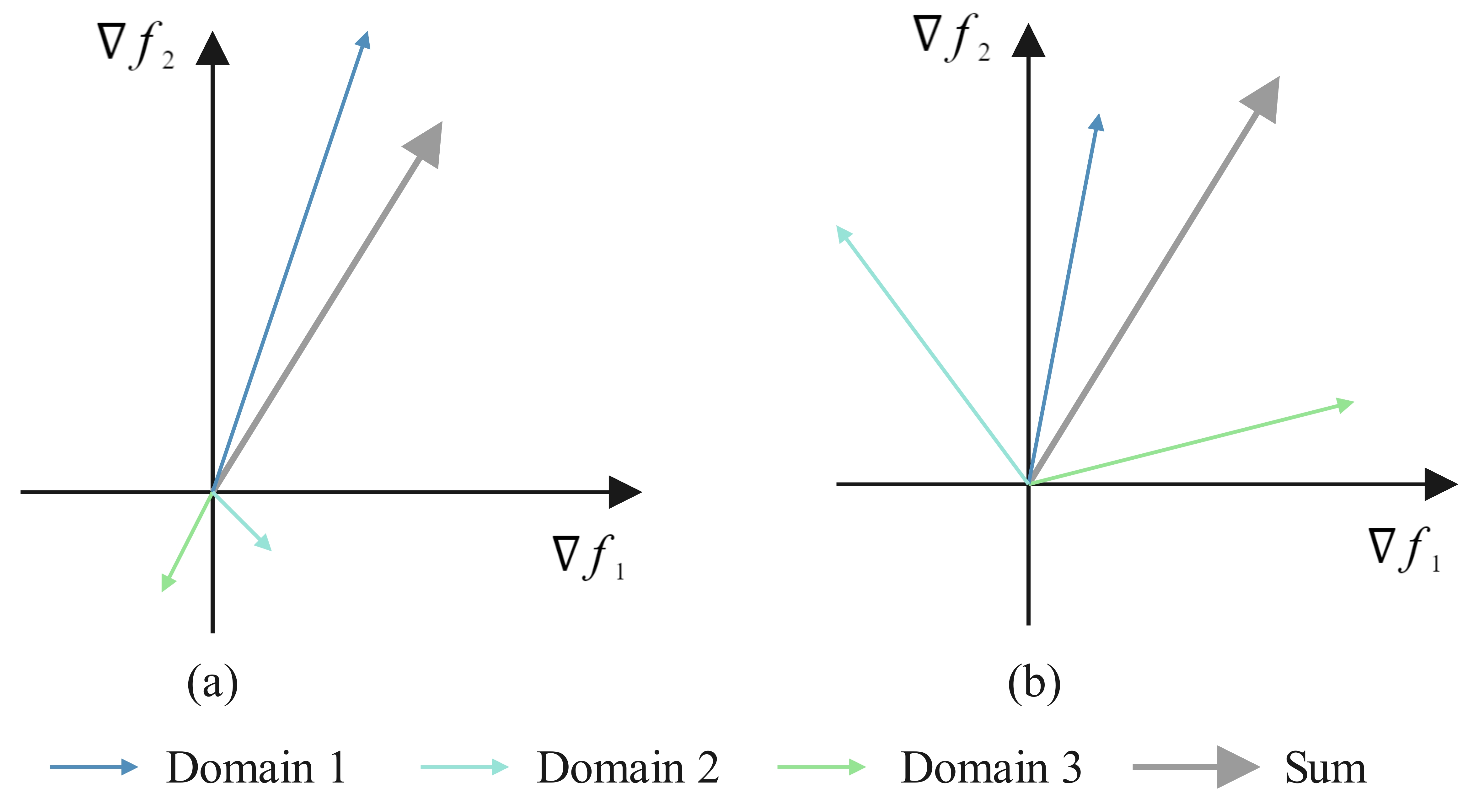}}
	\caption{The image illustrates the gradient of the parameters we aim to filter out. In part (a), the parameters have significant gradients in only one domain; in part (b), the parameters are significant in all domains but in different directions.}
	\label{il2}
\end{center}
\end{figure}

\subsubsection{Select Parameters with Significance Across All Domains}
The first term of $\mathcal{K}_2$ relates to minimum absolute gradients across domains. Therefore, we should consider  $C_{min}^{S_i}$ of each domain for selecting instead of their sum. We prove it using a proposition:
\begin{proposition}\label{prop3.3}
Let $G_i^j = \nabla_{\theta_j} R(A(\theta,S),S_i)$, where $\theta_j$ represents the $j$ parameter in $\theta$. If
$$\hat{M}_{jj}:=I[(\sum_{j'=1}^m\prod_{i=1}^N I(|G_i^j|\geq |G_i^{j'}|)\geq m -[m\rho])],$$
where $\hat{C}_{min}^{S^i}$ is obtained by $\hat{M}$, while $C_{min}^{S^i}$ is obtained through an arbitrary selection method. For some fixed $\rho$, we have:
\begin{equation}
\sum_{i=1}^N\frac{\beta^2}{\hat{C}_{min}^{{S}^i}+2(1-\rho)}\leq \sum_{i=1}^N\frac{\beta^2}{C_{min}^{{S}^i}+2(1-\rho)}.
\end{equation}
\end{proposition}
The proof is shown in Appendix \cref{prop3.3_proof}. $\hat{M}$ represents selecting the parameters with significant gradients across all domains, with the number of selected parameters does not exceed $[m\rho]$. As shown in \cref{il2}(a), other methods may select parameters that are significant in only one domain, which may rise $\mathcal{K}_2$.

\subsubsection{Select Parameters with Aligned Gradients}
$\mathcal{K}_2$ includes an $\mathcal{H}$-divergence term between $S$ and $S_i$. Some methods \cite{Fish, Fishr} suggest that the gradient direction should be close to the average gradient across all domains to reduce this divergence. Building on this idea, we show another proposition:

\begin{proposition}\label{prop3.4}
Let 
$$\overline{M}_{jj}:=I(\sum_{i,i'}^{i\neq i'}G_i^j\cdot G_{i'}^j\leq \frac{1}{m}\sum_{j'}^{m}(\sum_{i,i'}^{i\neq i'}G_i^{j'}\cdot G_{i'}^{j'}))$$
Assume an original selection matrix $M'$, whose $\mathcal{H}$ divergence is denoted as $\mathcal{H'}$. Let the $\mathcal{H}$ divergence resulting from the selection $\overline{M} \cdot M'$ be denoted as $\overline{\mathcal{H}}$. We have:
\begin{equation}
		\max\limits_{i}\frac{1}{2}\overline{\mathcal{H}}\nabla \overline{\mathcal{H}}(\hat{S_i},\hat{S})\leq \max\limits_{i}\frac{1}{2}{\mathcal{H'}}\nabla {\mathcal{H'}}(\hat{S_i},\hat{S}).
\end{equation}
\end{proposition}

The proof is shown in \cref{prop3.4_proof}. $\overline{M}$ aims to select parameters with stable gradient variations across domains. \cref{prop3.4} can be more intuitively understood through \cref{il2}(b), where the parameter is significant across all domains, but its gradient direction varies substantially. This variation suggests it could increase the prediction discrepancy across different domains, raising the upper bound in \cref{th1}. 

From the theorems and propositions, determining $M$ using $\hat{M} \cdot \overline{M}$ is ideal, but practical deployment poses challenges. We introduce two-step operators in the next section to approximate $\hat{M} \cdot \overline{M}$.  

Some methods, such as LoRA and Adapter, align with our bound by promoting sparse updates to leverage pre-trained knowledge for better generalization \cite{lora, PEGO, adapter}. While most are not for domain generalization, \cite{PEGO} resembles our framework by introducing tunable layers to specific model parts, reducing $\rho$, and optimizing the $\mathcal{H}$-divergence between domains. However, its full-layer update strategy violates \cref{prop3.3}, resulting in a suboptimal $\mathcal{K}_2$. We will compare these methods in \cref{accpa} and \cref{futher}.

\subsection{The Implementation of JPS\label{JPS2}}
We now present the overall implementation process, as shown in \cref{alg}.
 For some fixed $\rho$, we use $M^{step2}$ to approximate $\hat{M}\cdot \overline{M}$:
\begin{gather}
	\min_{\nabla\theta,M}R(h_{\theta_0+M^{step2}\nabla\theta},\hat{S}),\\
	s.t.\quad M^{step2} \approx \hat{M}\cdot\overline{M}, M^{step2}_{ij} = 0 ,M^{step2}_{ii}\in \{0,1\}.\nonumber \label{jps}
\end{gather}
We now briefly introduce how to generate $M^{step2}$. Our analysis and experiments focus on the CLIP \cite{clip} pre-trained ViT model \cite{vit}. We can retain more general knowledge by keeping more layers frozen, potentially improving generalization ability while enhancing sparsity \cite{lora, SPU}. Therefore, as shown in \cref{alg}, we only update the first linear layer in each ViT module, leaving the others frozen.

$M^{step2}$ is generated before training and remains unchanged once created. For clarity, we assume only one layer $\mathcal{L}_0$ is chosen for updating. This setting means that in the initial matrix $M$, values in other layers are zero, i.e., $M_{ii} = 0, \forall i \not\in \mathcal{L}_0$. A validation set $V = \{V_1, V_2, \dots, V_N\}$, where $V_i \subset \hat{S}_i$, is formed using a subset of the training data. We compute gradients on each $V_i$ using the pre-trained model with the task loss $R$:
\begin{equation}
	\nabla R^i_{\mathcal{L}_0} =M \cdot \nabla_{\theta_0} R(h_{\theta_0},V_i). \label{pret}
\end{equation}
The absolute gradient value represents parameter importance. With a hyperparameter $\rho$, we select the top $[m_{\mathcal{L}0}\rho]$ important parameters. These selections form domain-specific matrices $M_i$, combined element-wise to produce $M^{step1}$, as shown in part (a) of \cref{alg}.

Next, gradient variance is computed for $M^{step1}\cdot\nabla R_{\mathcal{L}_0}$. Parameters with unstable gradients are removed to reduce $\mathcal{H}\nabla \mathcal{H}(\hat{S_i}, \hat{S})$, yielding the final matrix $M^{step2}$. During training, only the parameters selected by $M^{step2}$ are updated, as shown in part (b) of \cref{alg}. This process can apply to other layers, computing gradients only up to the first selected layer to save resources. Task loss is the only loss during selection and training. Now we discuss the details of our implementation:

\subsubsection{MLP Layer Update}
In previous discussions, we emphasized that the proposed method selects parameters from specific layers with gradient-based selection. As shown in \cref{il1}, many parameters might be already close to the optimal solution for the target domains, making further training unnecessary. Hence, we argue that freezing enough layers is crucial to maintain sufficient sparsity $\rho$ while preserving parameters near the optimal solution. In the ViT architecture, linear layers have a significant impact than attention layers and require fewer resources for optimization \cite{select1,select2}. Additionally, studies \cite{select3, SPU} suggest that in attention networks, the first MLP layer serves as a memory key to detect patterns, while the second layer learns their distribution. To retain general knowledge while adapting to the domain generalization task, we update only the first MLP layer. In the initial $M$ matrix, all positions outside the selected layer are $0$:
\begin{equation}
	M_{ii} = 1,\forall i\in \mathcal{L}_0; M_{ii} = 0, \forall i\not\in \mathcal{L}_0;M_{ij}=0 ,\forall i\neq j.
\end{equation}
 We will further explore updates to different layers in \cref{futher} to validate our selection.

\subsubsection{Importance Selection Operator}
After determining the update layers, based on \cref{theorem3_2} and \cref{prop3.3}, we need to calculate $|\nabla R(A(\theta, S), S_i)|$ for parameter selection. However, since obtaining $|\nabla R(A(\theta, S), S_i)|$ is impractical, we assume that the pre-trained model $\theta_0$ is close to $A(\theta,S)$, similar to many other domain generalization methods \cite{miro, GES, PEGO}:
\begin{equation}
	 \nabla R(h_{\theta_0}, V_i) \approx \nabla R(A(\theta,S), S_i).
\end{equation}
 Therefore, we use the pre-trained model $\theta_0$ and its importance $|\nabla R(h_{\theta_0}, V_i)|$ in \cref{pret} as an approximation:
\begin{align}
	M_{i} = &M\cdot I(|\nabla R_{\mathcal{L}_0}^i|\geq |\nabla R_{{\mathcal{L}_0},[m_{\mathcal{L}_0}\rho]}^i| )\label{step11}, \\ 
	&M^{step1}=M_1\cdot M_2\cdot...M_N. \label{step1}
\end{align}
$\nabla R_{{\mathcal{L}_0},[m_{\mathcal{L}_0}\rho]}^i|$ represents the value of $[m_{\mathcal{L}0}\rho]$-th significant gradient. This operator selects the parameters with significant gradients across all domains to ensure we approximate the $\hat{M}$ in \cref{prop3.3}. As \cref{step11} shows, we select the parameters that are important in all source domains to approximate $\hat{M}$.

\subsubsection{Variance Selection Operator}
In the Importance Selection Operator, parameters are chosen based on the absolute value of their gradients. Let $\mathcal{L}_{0j}$ be the $j$-th parameter of $\mathcal{L}_0$. To approximate $\overline{M}$ in \cref{prop3.4}, as shown in part (b) of \cref{alg}, we compute the gradient variance for each parameter and apply a mean-variance threshold to discard those with high variance. Specifically, we calculate the gradient variances for the selected parameters:
\begin{equation}
	\sigma({\mathcal{L}_{0j}}) =\sum_{i=1}^N(\nabla R_{\mathcal{L}_{0j}}^i-\nabla \overline{R_{\mathcal{L}_{0j}}})^2, M_{jj}^{step1}=1,
\end{equation}
where $\nabla \overline{R_{\mathcal{L}_{0j}}}$ represents the average gradient of parameter $j$ across $N$ domains. $	{\sigma({\mathcal{L}_0})}$ is equivalent to $\sum_{j'=1}^mI(\sum_{i,i'}^{i\neq i'}G_i^j\cdot G_{i'}^j)$ with linear time calculation. To filter parameters with high variance, we apply a mean-variance threshold. The selection process can be expressed as:
\begin{align}
	\overline{\sigma({\mathcal{L}_0})}&=\frac{1}{||M^{step1}||_0}\sum_{j}\sigma({\mathcal{L}_0}_j)\label{step21},\\
	M^{step2}&=M^{step1}\cdot I(\sigma({\mathcal{L}_0})\leq \overline{\sigma({\mathcal{L}_0})}). \label{step2}
\end{align}
In this way, we approximate $\hat{M} \cdot \overline{M}$ using $M^{step2}$. We only use task loss to select and update the model. Experimental results will demonstrate that the chosen parameters are minimal but crucial, validating the effectiveness of our approach.

\section{Experiment}

\subsection{Experimental Setup}

\textbf{Datasets.} We utilize the Domainbed evaluation protocols \cite{domainbed} for a fair comparison across five benchmarks: PACS \cite{PACS} (4 domains \& 7 classes), VLCS \cite{VLCS} (4 domains \& 5 classes), OfficeHome \cite{Office} (4 domains \& 65 classes), TerraIncognita \cite{Terra} (4 domains \& 10 classes), and DomainNet \cite{Domain} (6 domains \& 345 classes).

\textbf{Evaluation protocol.} We adopt the experimental protocol of Domainbed \cite{domainbed}, which enforces fair and realistic evaluations (e.g., same model selection criterion) across competitors. For each experiment, we designate one domain as the target domain and use the remaining domains as the training set. Each domain is used as the target once. The training data comprises 80\% of the total training set, while the remaining 20\% is the validation set. Notably, the validation data for our method is drawn from the 80\% training data to prevent information leakage. We test our method on each dataset using three random seeds and report the final accuracy and variance results.

\textbf{Implementation details.} We use the CLIP pre-trained ViT-B/16 as the backbone, comprising 12 attention modules, and Adam as the optimizer. JPS has a single hyperparameter $\rho$, which controls parameter selection. For each dataset, we conduct a hyperparameter search with $\rho \in [0.2, 0.1, 0.05, 0.01, 0.005, 0.001, 0.0005, 0.0001]$, learning rates in [$8\times 10^{-5}$, $5\times 10^{-5}$], dropout values in $[0.8, 0.5, 0.3, 0.1, 0]$, and validation set sizes of $[10, 20, 50]$ times the batch size. The batch size varies by dataset, with weight decay set to 0. \cref{main} shows the generalization results of JPS with different layer updates. For instance, $L=12$ updates the first MLP layer of all 12 attention modules, while $L=8$ updates only the last 8. Hyperparameters for each dataset are in the Appendix \cref{details}. Our models are selected and trained using only cross-entropy loss.

\textbf{Baseline.} We compare our approach with several existing state-of-the-art domain generalization methods, particularly those that utilize pre-trained models to enhance generalization. The algorithms we compare include approaches using ResNet50 \cite{resnet} pre-trained on ImageNet \cite{ImageNet} and ViT-B/16 \cite{vit} pre-trained by Clip \cite{clip}. The results of this comparison are shown in \cref{main}. 

\begin{table*}[htpb]
	\centering
	\caption{Domain Generalization Compariosn. Present the best result in \textbf{bold}, and \underline{underline} the second-best result. }
	\label{main}
	{\begin{tabular}{@{}ccccccc@{}}
			\bottomrule
			\toprule
			Method &
			PACS &
			VLCS &
			OfficeHome &
			TerraInc &
			DomainNet &
			Avg. \\ \midrule
			\multicolumn{7}{c}{ResNet 50 pre-trained on ImageNet.} \\ \midrule
			\multicolumn{1}{l|}{ERM \cite{ermi}} &
			84.2 $\pm$ \small 0.1 &
			77.3 $\pm$ \small 0.1 &
			67.6 $\pm$ \small 0.2 &
			47.8 $\pm$ \small 0.6 &
			\multicolumn{1}{c|}{44.0 $\pm$ \small 0.1} &
			64.2 \\
			\multicolumn{1}{l|}{CORAL \cite{CORAL}} &
			86.2$\pm$ \small 0.3 &
			78.8 $\pm$ \small 0.6 &
			68.7$\pm$ \small 0.3 &
			47.6$\pm$ \small 1.0 &
			\multicolumn{1}{c|}{41.5 $\pm$ \small 0.1} &
			64.5 \\
			\multicolumn{1}{l|}{MIRO \cite{miro}} &
			85.4 $\pm$ \small 0.4 &
			79.0 $\pm$ \small 0.0 &
			70.5 $\pm$ \small 0.4 &
			50.4 $\pm$ \small 1.1 &
			\multicolumn{1}{c|}{44.3$\pm$ \small 0.2} &
			65.9 \\
			\multicolumn{1}{l|}{SAGM \cite{SAGM}} &
			86.6  $\pm$ \small 0.2 &
			80.0  $\pm$ \small 0.3 &
			70.1  $\pm$ \small 0.2 &
			48.8  $\pm$ \small 0.9 &
			\multicolumn{1}{c|}{45.0  $\pm$ \small 0.2} &
			66.1 \\
			\multicolumn{1}{l|}{GGA \cite{ballas2025gradient}} &
			87.3  $\pm$ \small 0.4 &
			79.9  $\pm$ \small 0.4 &
			68.5  $\pm$ \small 0.2 &
			50.6  $\pm$ \small 0.1 &
			\multicolumn{1}{c|}{45.2  $\pm$ \small 0.2} &
			66.3 \\
			\multicolumn{1}{l|}{SWAD \cite{swad}} &
			88.1 $\pm$ \small 0.1 &
			79.1 $\pm$ \small 0.1 &
			70.6$\pm$ \small 0.2 &
			50.0 $\pm$ \small 0.3 &
			\multicolumn{1}{c|}{46.5 $\pm$ \small 0.1} &
			66.9 \\
			\rowcolor{mycolor}
			\multicolumn{1}{l|}{ JPS } &
			89.6  $\pm$ \small 0.5 &
			82.1  $\pm$ \small 0.4 &
			70.3  $\pm$ \small 0.2 &
			48.4  $\pm$ \small 0.0 &
			\multicolumn{1}{c|}{46.6  $\pm$ \small 0.1} & 67.4
			\\
			\midrule
			\multicolumn{7}{c}{ViT-B/16 pre-trained with CLIP.} \\ \midrule
			\multicolumn{1}{l|}{ERM \cite{clip}} &
			93.7  $\pm$ \small 0.1 &
			82.7  $\pm$ \small 0.1 &
			78.5  $\pm$ \small 0.1 &
			52.3  $\pm$ \small 0.1 &
			\multicolumn{1}{c|}{53.8  $\pm$ \small 0.1} &
			72.2 \\
			\multicolumn{1}{l|}{MIRO \cite{miro}} &
			95.6  $\pm$ \small 0.8 &
			82.2  $\pm$ \small 0.3 &
			82.5  $\pm$ \small 0.1 &
			54.3  $\pm$ \small 0.4 &
			\multicolumn{1}{c|}{54.0  $\pm$ \small 0.3} &
			73.7 \\
			\multicolumn{1}{l|}{SPG \cite{softprompt}} &
			\textbf{97.0} $\pm$ \small 0.5 &
			82.4 $\pm$ \small 0.4 &
			83.6 $\pm$ \small 0.4 &
			50.2 $\pm$ \small 1.2 &
			\multicolumn{1}{c|}{60.1 $\pm$ \small 0.5} &
			74.7 \\
			\multicolumn{1}{l|}{GESTUR \cite{GES}} &
			96.0  $\pm$ \small 0.0 &
			82.8  $\pm$ \small 0.1 &
			84.2  $\pm$ \small 0.1 &
			55.7  $\pm$ \small 0.2 &
			\multicolumn{1}{c|}{\underline{58.9}  $\pm$ \small 0.1} &
			75.5 \\
			\multicolumn{1}{l|}{VL2V \cite{vl2v}} &
			94.3 $\pm$ \small 0.6 &
			82.3 $\pm$ \small 0.3 &
			\textbf{85.8} $\pm$ \small 0.2 &
			55.3 $\pm$ \small 0.7 &
			\multicolumn{1}{c|}{\textbf{59.2} $\pm$ \small 0.1} &
			75.5 \\
			\multicolumn{1}{l|}{DALSCLIP \cite{zhang2025dalsclip}} &
			96.7 $\pm$ \small 0.1 &
			\textbf{83.7} $\pm$ \small 0.3 &
			83.7 $\pm$ \small 0.3 &
			55.5 $\pm$ \small 0.2 &
			\multicolumn{1}{c|}{58.1 $\pm$ \small 0.0} &
			75.5 \\
			\rowcolor{mycolor}
			\multicolumn{1}{l|}{JPS $L = 12$} &
			95.9  $\pm$ \small 0.1 &
			83.2  $\pm$ \small 0.5 &
			84.4  $\pm$ \small 0.1 &
			\underline{57.9}  $\pm$ \small 0.6 &
			\multicolumn{1}{c|}{58.8  $\pm$ \small 0.0} &
			\underline{76.0} \\
			\rowcolor{mycolor}
			\multicolumn{1}{l|}{ JPS $L = 8$} &
			96.0  $\pm$ \small 0.0 &
			\textbf{83.7}  $\pm$ \small 0.2 &
			\underline{84.5}  $\pm$ \small 0.1 &
			\textbf{59.2}  $\pm$ \small 0.8 &
			\multicolumn{1}{c|}{58.7 $\pm$ \small 0.1} &
			\textbf{76.4} \\
			\rowcolor{mycolor}
			\multicolumn{1}{l|}{ JPS $L = 4$} &
			\underline{96.4}  $\pm$ \small 0.1 &
			\underline{82.9}  $\pm$ \small 0.3 &
			84.0  $\pm$ \small 0.1 &
			54.0$\pm$ \small 0.2 &
			\multicolumn{1}{c|}{58.0 $\pm$ \small 0.0
			} & 75.1 \\
			\rowcolor{mycolor}
			\multicolumn{1}{l|}{ JPS $L = 2$} &
			96.3  $\pm$ \small 0.1 &
			82.4  $\pm$ \small 0.0 &
			83.8  $\pm$ \small 0.0 &
			55.1$\pm$ \small 1.0 &
			\multicolumn{1}{c|}{58.4  $\pm$ \small 0.0} & 75.2
			\\
			\hline \bottomrule
	\end{tabular}}
\end{table*}

\begin{table*}[htbp]
			\caption{ Efficient Fine-tuning Comparison. The best results are highlighted in \textbf{bold}.}
				\centering
	\label{eff}
	\begin{tabular}{@{}l|cccccc|cc@{}}
		\bottomrule
		\toprule
		Method & \multicolumn{2}{c}{PACS} & \multicolumn{2}{c}{VLCS} & \multicolumn{2}{c|}{OfficeHome} & \multicolumn{2}{c}{Avg.} \\ \midrule
		& Acc.          & Tunable & Acc.          & Tunable & Acc.          & Tunable & Acc. & Tunable \\ \midrule
		ERM \cite{clip}        & 93.7  $\pm$ \small 0.1          & 86M     & 82.7 $\pm$ \small 0.1          & 86M     & 78.5 $\pm$ \small 0.1         & 86M     & 85.0 & 86M     \\
		MIRO \cite{miro}       & 95.6 $\pm$ \small 0.8         & 172M    & 82.2 $\pm$ \small 0.3          & 172M    & 82.5 $\pm$ \small 0.1         & 172M    & 86.8 & 172M    \\
		SAGM \cite{SAGM}       & 94.3 $\pm$ \small 0.3          & 86M     & 82.0  $\pm$ \small 0.3        & 86M     & 83.4   $\pm$ \small 0.1       & 86M     & 86.6 & 86M     \\
		GESTUR \cite{GES}      & 96.0 $\pm$ \small 0.0         & 172M    & 82.8 $\pm$ \small 0.1         & 172M    & 84.2  $\pm$ \small 0.1        & 172M    & 87.7 & 172M    \\
		Adapter \cite{adapter} & 92.0  $\pm$ \small 0.5        & 160K   & 79.8 $\pm$ \small 0.4          & 160K   & 72.9 $\pm$ \small 0.4          & 160K   & 81.6 & 160K   \\
		LoRA \cite{lora}       & 96.0 $\pm$ \small 0.1         & 150K   & 82.7  $\pm$ \small 0.0        & 300K   & 83.4  $\pm$ \small 0.1        & 300K   & 87.4 & 150K   \\
		PEGO \cite{PEGO}       & \textbf{96.5 $\pm$ \small 0.1} & 290K   & 83.2 $\pm$ \small 0.3         & 580K   & 84.2 $\pm$ \small 0.3         & 580K   & 87.9 & 480K   \\
		\rowcolor{mycolor}
		JPS            & 96.4 $\pm$ \small 0.1          & \textbf{1.5K}    & \textbf{83.7 $\pm$ \small 0.2} & \textbf{4.7K}    & \textbf{84.5 $\pm$ \small 0.1} & \textbf{0.5K }    &\textbf{ 88.2 } & \textbf{2.2K}    \\ \hline \bottomrule
	\end{tabular}
\end{table*}

\subsection{Comparison with Domain Generalization Methods}
Our approach enhances model generalization on unseen domains. The training loss in our method is the same as ERM, differing only in the pre-training parameter selection process. This selection increases average accuracy from 72.2 pp to 76.3 pp, with low variance across most datasets, demonstrating the effectiveness and stability of our method. Our algorithm achieves the best or second-best results on multiple datasets, highlighting the value of leveraging pre-trained model knowledge to achieve excellent performance without additional structures or assumptions.
 
Furthermore, reducing the number of training layers $L$ lowers prediction accuracy but maintains competitiveness. Given the ViT structure, fewer layers accelerate training and reduce consumption. These results show that our method ensures high efficiency and acceptable performance, making it suitable for resource-constrained scenarios.

\subsection{Comparison with Efficient Fine-tuning Methods\label{accpa}}

This experiment includes parameter-efficient methods with their results from \cite{PEGO}, where the number of tunable layers varies across datasets. Other settings are consistent with our main experiments. For comparison, we use the best accuracy results from the main experiment under aligned settings on the PACS, VLCS, and OfficeHome, reporting accuracy and tunable parameters. Tunable refers to the number of tunable parameters in the backbone. The results are shown in \cref{eff}.  For the complete comparison, please refer to the Appendix \cref{details}. 

Our method achieves notable performance with minimal parameter fine-tuning. Accuracy results are on $L=8$ (VLCS, OfficeHome) and $L=4$ (PACS). In contrast, many pre-trained model-based methods require tuning the entire model, adjusting twice as many parameters as ERM. This requirement can lead to excessive resource use and transmission challenges. For instance, transferring our model from PACS to OfficeHome requires only a sparse matrix with thousands of values, whereas others need the entire model.

Compared to parameter-efficient methods, our approach achieves higher accuracy with only a few thousand adjustable parameters. This result indicates that a few crucial parameters can significantly impact the prediction outcomes. It also highlights that pre-trained models encode enough information, and freezing most parameters ensures excellent generalization.

\subsection{Analysis Study\label{futher}}

\begin{figure*}[t]
	\includegraphics[width=1\textwidth]{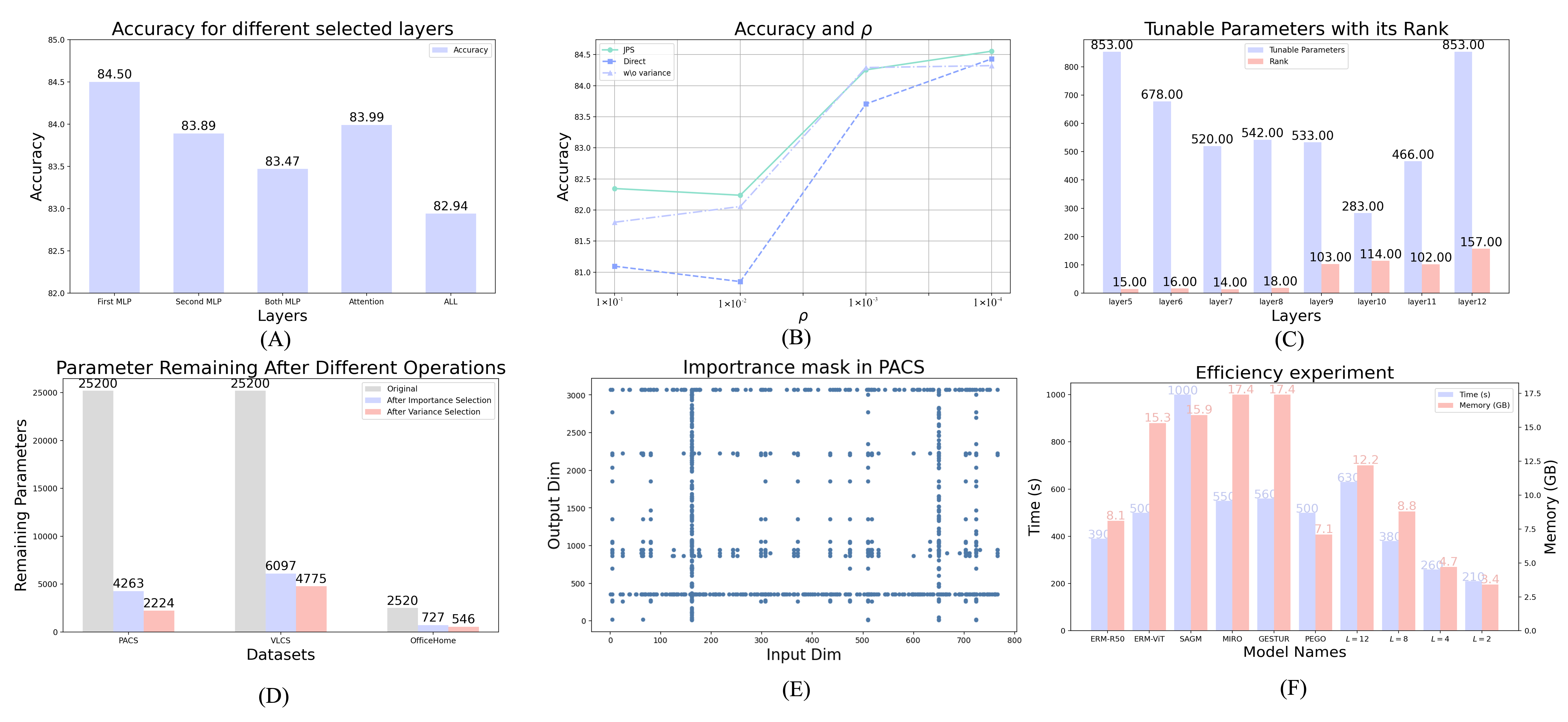}
	\caption{Analysis Study. The corresponding descriptions can be found in the paragraphs in \cref{accpa}.}
	\label{p5}
\end{figure*}

We further analyze the characteristics of our method, using JPS with $L=8$ and experimental settings consistent with the main experiments unless stated otherwise. The number of tunable parameters also represents those in the backbone.

\subsubsection{Accuracy for Different Selected Layers.} In \cref{JPS}, we discussed the rationale for selecting only the parameters from the first linear layer of ViT for training. To validate this choice, we conducted additional experiments on OfficeHome. All experimental hyperparameters are consistent with those used for $L=8$ in the main experiment, as shown in part (A) of \cref{p5}. The subscripts in the figure represent the layers for parameter selection, with "ALL" indicating that all layers, except the LN layers, are involved in parameter selection and updates. The experimental results show that fine-tuning the first MLP layer yields the best accuracy, while selecting other layers leads to performance degradation, validating the rationale behind choosing the first linear layer as the target for JPS.

\begin{table}[htbp]
	\centering
	\caption{Accuracy for JPS $L = 8$ under Different $\rho$. \label{rho}}
	% \resizebox{0.65\textwidth}{!}{
		\begin{tabular}{@{}cccccc@{}}
		\bottomrule
		\toprule
		$\rho$                            & 0.1    & 0.01   & 0.001  & 0.0001 & 0.00001 \\ \midrule
		\multicolumn{1}{c|}{JPS}          & 82.4 & 82.2 & 84.3 & 84.5 & 83.3  \\
		\multicolumn{1}{c|}{Direct}       & 81.0 & 80.8 & 83.7 & 84.3 & 84.0  \\
		\multicolumn{1}{c|}{w/o variance} & 81.8 & 82.1 & 84.3 & 84.3 & 83.4  \\
		\hline \bottomrule
	\end{tabular}
	% }
\end{table}

\subsubsection{Accuracy for Different {$\mathbf{\rho}$}.}
This experiment evaluates the impact of two operators in our method and the effect of $\rho$ on model generalization. We updated the first linear layer from the last eight layers, keeping hyperparameters consistent with $L=8$ on OfficeHome, as shown in part (B) of \cref{p5}. "JPS" denotes our method, "Direct" selects parameters based on gradient sums across domains, and "w/o variance" uses only the importance selection operator. Results show that reducing $\rho$ improves target domain accuracy, indicating that freezing enough parameters to preserve general knowledge enhances generalization. While all methods improve as $\rho$ decreases, "JPS" outperforms both "Direct" and "w/o variance," with the latter surpassing "Direct," demonstrating the significance of our importance selection operator and the added benefit of variance selection. Even at $\rho = 1\times 10^{-4}$, JPS reduces parameters compared to "Direct" (see part (D) of \cref{p5}) while maintaining performance, and lowering transmission and training costs. Detailed values are provided in the \cref{rho}.

% 这句话有说清楚时掩码的秩吗？
\subsubsection{Tunable Parameters with its Rank.} In part (C) of \cref{p5}, we present the number of parameters and their ranks in the selection mask $M^{step2}$ on the PACS dataset. The ranks are computed using SVD. The results reveal a U-shaped distribution for the number of selected parameters, with corresponding ranks increasing. This finding contrasts with the low-rank structure assumed by adapter methods, suggesting that key parameters for downstream tasks may be sparse rather than low-rank and potentially have higher ranks.

\subsubsection{Parameters Remaining After Operators.} We display the number of parameters remaining after each step of the JPS method on different datasets in part (D) of \cref{p5}. "Original" refers to the parameter count for the eight linear layers multiplied by $\rho$ for each dataset. The results indicate that, for all datasets, the first operator significantly reduces the number of selected parameters, with reductions exceeding 50\%. The reduction suggests that significant parameters across all domains are relatively few. The second step further filters about 10\% parameters,  indicating that some parameters are indeed significant across all domains but differ in direction, validating the effectiveness of our method.

\subsubsection{Importance Mask.} We visualize the importance mask of the first linear layer in the last module of ViT on the PACS dataset in part (E) of \cref{p5}. It is important to note that the selection process depends solely on $\rho$, which ensures consistency across all models. Points marked represent selected parameter positions. The visualization shows that these selected parameters are concentrated along horizontal and vertical lines. This pattern suggests that the model identifies information important across all domains, likely corresponding to specific regions in the input features.

\subsubsection{Efficiency experiment.} This experiment evaluates resource consumption and training time. All settings match the main experiment, and we report average resource usage and time across five benchmarks. Metrics include peak GPU memory usage and time per 1,000 training steps. Memory represents the maximum GPU consumption per batch update on the Domainbed benchmark, while time reflects the average computation time per 1,000 updates. Testing is conducted on the same 4090 GPU to avoid device-related discrepancies. As shown in part (F) of \cref{p5}, with $L = 12$, our method slightly reduces resource usage but increases computation time due to gradient computation and updates within masks. However, as $L$ decreases, resource and time requirements drop while maintaining good generalization. These results demonstrate the efficiency of our method in leveraging pre-trained model knowledge.

\section{Conclusion}
We propose a new perspective: pre-trained models hold valuable generic knowledge, and fully fine-tune them risks losing this knowledge, reducing generalization. To address this, we introduce a framework for selecting and updating parameters, fine-tuning those with strong generalization ability from specific layers while preserving generic knowledge. Theoretically, we derive an error bound for domain generalization based on parameter sparsity, showing that selective fine-tuning reduces prediction error on unseen domains. In deployment, our method uses importance and variance selection operators to update key parameters. Experiments on Domainbed demonstrate that our method surpasses state-of-the-art domain generalization and parameter-efficient models. Further analysis reveals that frozen parameters retain knowledge while tunable ones enhance generalization, reinforcing the validity of our approach.

\bibliography{Reference}
\bibliographystyle{unsrt}
%%%%%%%%%%%%%%%%%%%%%%%%%%%%%%%%%%%%%%%%%%%%%%%%%%%%%%%%%%%%%%%%%%%%%%%%%%%%%%%
%%%%%%%%%%%%%%%%%%%%%%%%%%%%%%%%%%%%%%%%%%%%%%%%%%%%%%%%%%%%%%%%%%%%%%%%%%%%%%%
% APPENDIX
%%%%%%%%%%%%%%%%%%%%%%%%%%%%%%%%%%%%%%%%%%%%%%%%%%%%%%%%%%%%%%%%%%%%%%%%%%%%%%%
%%%%%%%%%%%%%%%%%%%%%%%%%%%%%%%%%%%%%%%%%%%%%%%%%%%%%%%%%%%%%%%%%%%%%%%%%%%%%%%
\newpage
\onecolumn
\appendices

\section{Proof of \cref{lemma3.1}\label{lemma3.1_proof}}
\begin{lemma}
Assume that the loss function \( R \) exhibits \( p \)-Lipschitz continuity and Pointwise Hypothesis Stability \( c \) \cite{stable} for method \( A \). Set \( |\nabla_{\theta} R(A(\theta, S), S)| \geq C_{min}^{S} \cdot I \). Let \( \beta = \max(p, c) \), we have:
\begin{align}
		E_{\hat{S},i\sim U(n)}&(|l(h_{A(\theta,\hat{S})}(x_i),y_i)-l(h_{A(\theta,\hat{S}^{/i})}(x_i),y_i)|)\leq \mathcal{K}_1,\nonumber\\
		&\mathcal{K}_1 = \frac{\beta^2}{(C_{min}^{\hat{S}}+2(1-\rho))n}, (x_i,y_i)\in \hat{S}.
\end{align}
\end{lemma}
\begin{proof}
First, set $\theta = \theta_0+M\nabla\theta$, we can reformulate \cref{spas} as:
\begin{gather}
	\min_{\theta}R(h_{\theta},S)\nonumber,\\
	s.t. ||(I-M)(\theta-\theta_0)||^2 =0. \label{spas1}
\end{gather}
Here, $M_{ii}\in\{0,1\}$, by Lagrangian duality, the optimization of \cref{spas1} is equal to solving the optimization with sparsity constraints:
\begin{equation}
	\min_{\theta}\max_{\lambda}R(h_{\theta},S)+\lambda ||(I-M)(\theta-\theta_0)||^2. \label{spas2}
\end{equation}
It can be proven that the solution to \cref{spas2} is the same as the optimization of \cref{spas1} \cite{onthe}. We formally introduce pointwise stability here:

\textbf{Pointwise Hypothesis Stability.} \cite{stable}  An algorithm A has pointwise hypothesis stability $c$ with respect to the loss function $R$ if the following holds:
\begin{equation}
	E_{\hat{S},i\sim U(n)}|l(h_{A(\theta,\hat{S})}(x_i), y_i) - l(h_{A(\theta,\hat{S}^{/i})}(x_i), y_i)| \leq c \label{pf4}
\end{equation}
 
Given that we are performing a classification task with parameters trained by $A$, we assume the function satisfies this condition for some $c$.

We simplify the above optimization objective as follows:
\begin{equation}
	f_{\hat{S}}(\theta)=R(h_{\theta},\hat{S})+\lambda ||(I-M)(\theta-\theta_0)||^2 \label{pf1}.
\end{equation}

Based on the definition in our main text, let us assume the model parameters learned on $\hat{S}$ are represented as $A(\theta, \hat{S})$. At this point, we can express $f_{\hat{S}}(\theta)$ using the Taylor expansion in $A(\theta, \hat{S})$  as follows:
\begin{equation}
	f_{\hat{S}}(\theta) = f_{\hat{S}}(A(\theta,\hat{S}))+(\theta-A(\theta,\hat{S}))^T\nabla f_{\hat{S}}(A(\theta,\hat{S})) \label{pf2}
\end{equation}
We assume that the $A(\theta, \hat{S})$ learned under the $A$ method is obtained by minimizing the loss function, but we assume it has a minimal non-zero gradient, which is typical in most real-world scenarios:
\begin{equation}
	|\nabla R(A(\theta,\hat{S}),\hat{S})| \geq C_{\min}^{\hat{S}}\cdot I \label{pf3}.
\end{equation}
Here, the inequality indicates that the value at any position in the gradient is greater than $C_{min}^{\hat{S}}$, meaning each component of the gradient is at least as large as $C_{min}^{\hat{S}}$. Using definition in \cref{pf1} to express \cref{pf2}:
\begin{equation}
	f_{\hat{S}}(\theta)-f_{\hat{S}}(A(\theta,\hat{S})) =\\
	(\theta-A(\theta,\hat{S}))^T\nabla (R(A(\theta,\hat{S}),\hat{S})\\
	+\lambda||(I-M)(\theta-\theta_0)||^2).
\end{equation}
Using the inequality in \cref{pf3}:
\begin{equation}
	f_{\hat{S}}(\theta)-f_{\hat{S}}(A(\theta,\hat{S}))\geq(\theta-A(\theta,\hat{S}))^T(C_{\min}^{\hat{S}}+2\lambda )\cdot I.
\end{equation}
Then we have:
\begin{equation}
	f_{\hat{S}}(\theta)-f_{\hat{S}}(A(\theta,\hat{S}))\geq(C_{\min}^{\hat{S}}+2\lambda )||\theta-A(\theta,\hat{S})||.
\end{equation}

Now, we consider $\forall \theta_1,theta_2 \in \Theta$:
\begin{equation}
	f_{\hat{S}}(\theta_1)-f_{\hat{S}}(\theta_2) = \\
	R(\theta_1,\hat{S})+\lambda||\theta_1-\theta_0||^2\\
	-( R(\theta_2,\hat{S})+\lambda||\theta_2-\theta_0||^2)
\end{equation}
We randomly select a sample point $(x_i, y_i)$ from $\hat{S}$, and as described in the main text, use $\hat{S}^{/i}$ to denote the remaining sample set after removing $(x_i, y_i)$ from $\hat{S}$. The above equation can be rewritten as:
\begin{equation}
	f_{\hat{S}}(\theta_1)-f_{\hat{S}}(\theta_2) =\\
	R(\theta_1,\hat{S}^{/i})+\lambda||\theta_1-\theta_0||^2\\
	-( R(\theta_2,\hat{S}^{/i})+\lambda||\theta_2-\theta_0||^2)\\
	+\frac{l(h_{\theta_1}(x_i),y_i)-l(h_{\theta_2}(x_i),y_i)}{n}
\end{equation}
Let $\theta_1 = A(\theta, \hat{S}^{/i})$ and $\theta_2 = A(\theta, \hat{S})$. Note that $A(\theta, \hat{S}^{/i})$ minimizes $R(\theta_1, \hat{S}^{/i}) + \lambda ||\theta - \theta_0||^2$:
\begin{equation}
	f_{\hat{S}}(A(\theta,\hat{S}))-f_{\hat{S}}(A(\theta,\hat{S}^{/i}))\leq\\
	\frac{l(h_{A(\theta,\hat{S})}(x_i),y_i)-l(h_{A(\theta,\hat{S}^{/i})}(x_i),y_i)}{n}.
\end{equation}
Using the assumption in \cref{pf4}, we have:
\begin{equation}
	E_{\hat{S},i\sim U(n)}(C_{min}^{\hat{S}}+2\lambda)||A(\theta,\hat{S}^{/i})-A(\theta,\hat{S})||\leq\\
	E_{\hat{S},i\sim U(n)} \frac{l(h_{A(\theta,\hat{S})}(x_i),y_i)-l(h_{A(\theta,\hat{S}^{/i})}(x_i),y_i)}{n},
\end{equation}
\begin{equation}
	E_{\hat{S},i\sim U(n)}(C_{min}^{\hat{S}}+2\lambda)||A(\theta,\hat{S}^{/i})-A(\theta,\hat{S})||\leq \frac{c}{n}.
\end{equation}
Thus we have:
\begin{equation}
		E_{\hat{S},i\sim U(n)}||A(\theta,\hat{S}^{/i})-A(\theta,\hat{S})||\leq \frac{c}{n(C_{min}^{\hat{S}}+2\lambda)} \label{pf6}
\end{equation}
Using the $p$-Lipschitz condition and assuming $\beta = \max(c, p)$, we can derive the following equation:
\begin{equation}
	l(h_{A(\theta,\hat{S})}(x_i),y_i)-l(h_{A(\theta,\hat{S}^{/i})}(x_i),y_i)\\
	\leq p||A(\theta,\hat{S}^{/i})-A(\theta,\hat{S})||,
\end{equation}
\begin{equation}
		E_{\hat{S},i\sim U}(|l(h_{A(\theta,\hat{S})}(x_i),y_i)-l(h_{A(\theta,\hat{S}^{/i})}(x_i),y_i)|)\\
		\leq \frac{\beta^2}{n(C_{min}^{\hat{S}}+2\lambda)} \label{pf5}
\end{equation}
At this point, assuming we have set the sparsity of the updates as $\rho$, the following equation holds:
\begin{equation}
\begin{aligned}
	&E||(I-M)(\theta-\theta_0)||
	\\=&E\sum_{i=1}^n(1-M_{ii})^2(\theta^i-\theta^i_0)^2
	\\=&\sum_{i=1}^n(\theta^i-\theta^i_0)^2E(1-M_{ii})^2
	\\=&\sum_{i=1}^n(\theta^i-\theta^i_0)^2(1-\rho).
\end{aligned}
\end{equation}
Thus in $A(\theta,\hat{S})$, we have $\lambda = 1-\rho$. Combine this result with \cref{pf5}, we proof the \cref{lemma3.1}:
\begin{equation}
	E_{\hat{S},i\sim U(n)}(|l(h_{A(\theta,\hat{S})}(x_i),y_i)-l(h_{A(\theta,\hat{S}^{/i})}(x_i),y_i)|)\\
	\leq \frac{\beta^2}{(C_{min}^{\hat{S}}+2(1-\rho))n}.
\end{equation}
\end{proof}

\section{Proof of \cref{prop3.3}\label{prop3.3_proof}}
\begin{proposition}
Let $G_i^j = \nabla_{\theta_j} R(A(\theta,S),S_i)$, where $\theta_j$ represents the $j$ parameter in $\theta$. If
$$\hat{M}_{jj}:=I[(\sum_{j'=1}^m\prod_{i=1}^N I(|G_i^j|\geq |G_i^{j'}|)\geq m -[m\rho])],$$
where $\hat{C}_{min}^{S^i}$ is obtained by $\hat{M}$, while $C_{min}^{S^i}$ is obtained through an arbitrary selection method. For some fixed $\rho$, we have:
\begin{equation}
\sum_{i=1}^N\frac{\beta^2}{\hat{C}_{min}^{{S}^i}+2(1-\rho)}\leq \sum_{i=1}^N\frac{\beta^2}{C_{min}^{{S}^i}+2(1-\rho)}.
\end{equation}
\end{proposition}
$\hat{M}$ represents selecting the parameters with significant gradients across all domains, with the number of selected parameters does not exceed $[m\rho]$. As shown in \cref{il2}(a), other methods may select parameters that are significant in only one domain, which may rise $\mathcal{K}_2$.
\begin{proof}
The proof of \cref{prop3.3} is straightforward. For a fixed $\rho$, according to
\begin{equation}
		\hat{M}_{jj}=I[(\sum_{j'=1}^m\prod_{i=1}^N I(|G_i^j|\geq |G_i^{j'}|)\geq m -[m\rho])].
\end{equation}
We select parameters that have significant gradients across all domains, with numbers not greater than $[m\rho]$. Therefore, compared to any other selection method, we have:
\begin{equation}
	\hat{C}_{min}^{{S}^i}\geq C_{min}^{{S}^i}.
\end{equation}
Thus, for a fixed $\rho$, the following equation holds:
\begin{equation}
	\frac{1}{N}\sum_{i=1}^N\frac{\beta^2}{\hat{C}_{min}^{{S}^i}+2(1-\rho)}\leq \frac{1}{N}\sum_{i=1}^N\frac{\beta^2}{C_{min}^{{S}^i}+2(1-\rho)}.
\end{equation}
\end{proof}

\section{Proof of \cref{prop3.4}\label{prop3.4_proof}}
\begin{proposition}
Let 
$$\overline{M}_{jj}:=I(\sum_{i,i'}^{i\neq i'}G_i^j\cdot G_{i'}^j\leq \frac{1}{m}\sum_{j'}^{m}(\sum_{i,i'}^{i\neq i'}G_i^{j'}\cdot G_{i'}^{j'}))$$
Assume an original selection matrix $M'$, whose $\mathcal{H}$ divergence is denoted as $\mathcal{H'}$. Let the $\mathcal{H}$ divergence resulting from the selection $\overline{M} \cdot M'$ be denoted as $\overline{\mathcal{H}}$. We have:
\begin{equation}
		\max\limits_{i}\frac{1}{2}\overline{\mathcal{H}}\nabla \overline{\mathcal{H}}(\hat{S_i},\hat{S})\leq \max\limits_{i}\frac{1}{2}{\mathcal{H'}}\nabla {\mathcal{H'}}(\hat{S_i},\hat{S}).
\end{equation}
\end{proposition}
\begin{proof}
The objective of $\overline{M}$ is to filter out parameters with large gradient variances, which aligns closely with the goal in \cite{Fish, Fishr}. In \cite{Fishr}, the authors demonstrated the effectiveness by optimizing for gradient variances, targeting the high-variance parts of the gradient. Here, we briefly reference their conclusion to show that our goal is consistent with the detailed process available in the original paper:

\begin{assumption}\label{assumption1}
	Suppose there exists an optimal solution parameter $\theta^*$ on $S \cup T = D$, which satisfies the following for any domain:
	\begin{equation}
		R(h_{\theta},S_i) = R(h_{\theta^*},S_i)+\frac{1}{2}(\theta-\theta^*)^TH_i(\theta-\theta^*),
	\end{equation}
	where all the eigenvalues of $H_i$ are greater than 0. Then we define the following loss:
	\begin{equation}
		I^{\epsilon}(\theta^*) = \max\limits_{(i,j)}\max\limits_{\theta\in N^{\epsilon}_{S_i,\theta^*}}|R(h_{\theta^*},S_i)-R(h_{\theta},S_j)|,
	\end{equation}
\end{assumption}
 
where $N^{\epsilon}_{S_i,\theta^*}$ if there exists a path in the weights space between $\theta$ and$\theta^*$ where the risk $R(h_{\theta^*},S_i)$ remains in an $\epsilon \geq 0$  interval around $N^{\epsilon}_{S_i,\theta^*}$. This equation measures the stability of a local optimum across different domains. The smaller $I^{\epsilon}(\theta^*)$ is the better the model's stability.

\begin{lemma}
	(Lemma 1 from \cite{Fishr}.) Under \cref{assumption1} with some small $\delta$, we have the following equation holds:
	\begin{equation}
		I^{\epsilon}(\theta^*) = \max\limits_{i,j}(R(h_{\theta^*},S_i)-R(h_{\theta^*},S_j)\\
		+\max\limits_{\frac{1}{2}\theta^T H_i \theta\leq \delta}\frac{1}{2}\theta^T H_j \theta)
	\end{equation}
\end{lemma}

This lemma provides a transformation path, showing that the stability of the model can be controlled by its Hessian matrices across different domains. Specifically, it can be expressed as:
\begin{equation}
\begin{aligned}
	\operatorname*{max}_{\frac{1}{2}\theta^{\top}H_{i}\theta\leq\delta}\frac{1}{2}\theta^{\top}H_{j}\theta & =\max_{\|\tilde{\theta}\|_{2}^{2}\leq\delta}\sum_{k}\tilde{\theta}_{k}^{2}\lambda_{k}^{j}/\lambda_{k}^{i} \\
	& =\epsilon\times\max_i\lambda_k^j/\lambda_k^i.
\end{aligned}
\end{equation}
This expression indicates that by controlling the similarity of gradients across different domains, we can achieve a smaller ratio of $\lambda_k^j / \lambda_k^i$, thereby reducing the model's instability. Unlike the method in \cite{Fishr}, we employ the mask $\overline{M}$ to directly filter out unstable parameters, specifically those with high variance, which helps reduce $\lambda_k^j / \lambda_k^i$. However, it is important to note that this filtering approach ensures better stability than the original method, though it may not be optimal.

Additionally, we observe that the formula includes a term representing the loss difference across domains. According to the proof in \cite{Fishr}, our filtering method also reduces this loss from a feature perspective. Thus, when considering the original selection matrix $M$ and the new matrix $\overline{M}$, the loss $I^{\epsilon}(\theta_{M})$ is reduced to $I^{\epsilon}(\theta_{M \cdot \overline{M}})$. Based on the definition of $I$, we can conclude that the $\mathcal{H}$-divergence has reduced, as the difference between the model prediction distributions across domains is minimized (By sampling the corresponding $h$ from $N^{\epsilon}{S_i, \theta_M}$, we can construct $\mathcal{H}$). Since we are training by selecting parameters, both $\theta_M$ and $\theta_{M \cdot \overline{M}}$ satisfy the condition. This assumption aligns with the definition of H-divergence:
\begin{equation}
	\max\limits_{i}\frac{1}{2}\overline{\mathcal{H}}\nabla \overline{\mathcal{H}}({S_i},{S})\leq \max\limits_{i}\frac{1}{2}{\mathcal{H'}}\nabla {\mathcal{H'}}({S_i},{S}).
\end{equation}
Given that the data is sampled from a distribution, we have the \cref{prop3.4} holds :
\begin{equation}
	\max\limits_{i}\frac{1}{2}\overline{\mathcal{H}}\nabla \overline{\mathcal{H}}(\hat{S_i},\hat{S})\leq \max\limits_{i}\frac{1}{2}{\mathcal{H'}}\nabla {\mathcal{H'}}(\hat{S_i},\hat{S}).
\end{equation}
\end{proof}

\section{Details in Experiments\label{details}}

\subsection{Hyperparameter in Main Experiment\label{hyperparameters}}
We provide the optimal combinations of these parameters on the JPS in the table below for easy reproducibility. The remaining parameters not mentioned are the default values used in Domainbed\cite{domainbed} with weight decay fixed with 0. All experiments are conducted on a single NVIDIA RTX 4090 GPU. Note that our method adjusts the number of layers, but for different values of $L$, we use the same hyperparameters for each dataset. This choice is for simplicity. Specifically, after searching for the hyperparameters under $L=12$, we directly apply them to the cases where $L=8$, $L=4$, and $L=2$. Python: 3.8.17, PyTorch: 2.0.1, Torchvision: 0.15.2, CUDA: 11.7, CUDNN: 8500, NumPy: 1.22.0, PIL: 9.4.0.
\begin{table}[!ht]
	\centering
	\caption{Hyperparameter in Main Experiment}
	% \resizebox{0.85\textwidth}{!}{
		\begin{tabular}{@{}cccccc@{}}
		\bottomrule
		\toprule
		\multirow{2}{*}{Dataset} &
		\multirow{2}{*}{learning rate} &
		\multirow{2}{*}{dropout rate} &
		\multirow{2}{*}{$\rho$} &
		\multirow{2}{*}{validation set size} &
		\multirow{2}{*}{Batch size} \\
		&                        &                      &                         &                     &                     \\ \midrule
		\multicolumn{1}{c|}{\multirow{2}{*}{PACS}}       & \multirow{2}{*}{8e-05} & \multirow{2}{*}{0.5} & \multirow{2}{*}{0.001}  & \multirow{2}{*}{10} & \multirow{2}{*}{32} \\
		\multicolumn{1}{c|}{}                            &                        &                      &                         &                     &                     \\
		\multicolumn{1}{c|}{\multirow{2}{*}{VLCS}}       & \multirow{2}{*}{8e-05} & \multirow{2}{*}{0.3} & \multirow{2}{*}{0.001}  & \multirow{2}{*}{10} & \multirow{2}{*}{32} \\
		\multicolumn{1}{c|}{}                            &                        &                      &                         &                     &                     \\
		\multicolumn{1}{c|}{\multirow{2}{*}{OfficeHome}} & \multirow{2}{*}{8e-05} & \multirow{2}{*}{0}   & \multirow{2}{*}{0.0001} & \multirow{2}{*}{20} & \multirow{2}{*}{32} \\
		\multicolumn{1}{c|}{}                            &                        &                      &                         &                     &                     \\
		\multicolumn{1}{c|}{\multirow{2}{*}{TerraIncognita}} &
		\multirow{2}{*}{5e-05} &
		\multirow{2}{*}{0.8} &
		\multirow{2}{*}{0.2} &
		\multirow{2}{*}{50} &
		\multirow{2}{*}{32} \\
		\multicolumn{1}{c|}{}                            &                        &                      &                         &                     &                     \\
		\multicolumn{1}{c|}{\multirow{2}{*}{DomainNet}}  & \multirow{2}{*}{8e-05} & \multirow{2}{*}{0.1} & \multirow{2}{*}{0.0005} & \multirow{2}{*}{10} & \multirow{2}{*}{20} \\
		\multicolumn{1}{c|}{}                            &                        &                      &                         &                     &                     \\ \hline \bottomrule
	\end{tabular}
	% }
\end{table}

\subsection{Main Results with More Details\label{main_detail}}

To evaluate generalization, we designate one domain as the test set while splitting the remaining domains into training and validation sets (80\% for training and 20\% for validation). The larger portion is for training, and the smaller one serves as the validation set. We report the test accuracy of the model that achieves the highest validation accuracy, along with the corresponding validation accuracy. For instance, in the PACS dataset, we perform four experiments: in each, three domains are for training and validation, and the remaining domain is for testing (e.g., training on art, cartoon, and photo, while testing on sketch). The target and validation accuracies are averaged across all domains to calculate the overall accuracy. Each experiment is trained for 5000 steps, with validation and test accuracies logged every 200 steps. This process is repeated with five fixed random seeds, and the mean and variance of the target and validation accuracies are calculated and reported in the Main Experiment. 

Here, we provide additional details about the primary experiments, focusing on supplementary data sources and a more comprehensive comparison with parameter-efficient methods. Since parameter-efficient methods treat the number of tunable layers as a variable, which creates an inconsistency with our setup, we further introduce "JPS Best." This variant also considers the number of tunable layers as a variable and selects the best results from the different $L$ settings.

† indicates that the data for this method comes from \cite{GES}. ‡ indicates that the data for this method comes from \cite{PEGO}. Others are from the original paper \cite{vl2v}. Both methods adhere to the experimental setup requirements defined by us and DomainBed.

\begin{table}[!ht]
	\renewcommand{\arraystretch}{1.5}
	\normalsize
	\centering
	\caption{The main result  with more details. Present the best result in \textbf{bold}.}
	% \resizebox{0.8\textwidth}{!}{
		\begin{tabular}{@{}ccccccc@{}}
			\toprule \bottomrule
			Method &
			PACS &
			VLCS &
			OfficeHome &
			TerraInc &
			DomainNet &
			Avg. \\ \midrule
			\multicolumn{7}{c}{ResNet 50   pre-trained on ImageNet.} \\ \midrule
			\multicolumn{1}{c|}{ERM\textsuperscript{†}} &
			84.2 $\pm$ \small 0.1 &
			77.3 $\pm$ \small 0.1 &
			67.6 $\pm$ \small 0.2 &
			47.8 $\pm$ \small 0.6 &
			\multicolumn{1}{c|}{44.0 $\pm$ \small 0.1} &
			64.2 \\
			\multicolumn{1}{c|}{CORAL\textsuperscript{†}} &
			86.2$\pm$ \small 0.3 &
			78.8 $\pm$ \small 0.6 &
			68.7$\pm$ \small 0.3 &
			47.6$\pm$ \small 1.0 &
			\multicolumn{1}{c|}{41.5 $\pm$ \small 0.1} &
			64.5 \\
			\multicolumn{1}{c|}{MIRO\textsuperscript{†}} &
			85.4 $\pm$ \small 0.4 &
			79.0 $\pm$ \small 0.0 &
			70.5 $\pm$ \small 0.4 &
			50.4 $\pm$ \small 1.1 &
			\multicolumn{1}{c|}{44.3$\pm$ \small 0.2} &
			65.9 \\
			\multicolumn{1}{c|}{SAGM\textsuperscript{†}} &
			86.6  $\pm$ \small 0.2 &
			80.0  $\pm$ \small 0.3 &
			70.1$\pm$ \small 0.2 &
			48.8  $\pm$ \small 0.9 &
			\multicolumn{1}{c|}{45.0$\pm$ \small 0.2} &
			66.1 \\
			\multicolumn{1}{c|}{SWAD\textsuperscript{†}} &
			88.1 $\pm$ \small 0.1 &
			79.1 $\pm$ \small 0.1 &
			70.6$\pm$ \small 0.2 &
			50.0 $\pm$ \small 0.3 &
			\multicolumn{1}{c|}{46.5 $\pm$ \small 0.1} &
			66.9 \\ \midrule
			\multicolumn{7}{c}{ViT-B/16 pre-trained with CLIP.} \\ \midrule
			\multicolumn{1}{c|}{ERM\textsuperscript{†}} &
			93.7  $\pm$ \small 0.1 &
			82.7  $\pm$ \small 0.1 &
			78.5  $\pm$ \small 0.1 &
			52.3  $\pm$ \small 0.1 &
			\multicolumn{1}{c|}{53.8  $\pm$ \small 0.1} &
			72.2 \\
			\multicolumn{1}{c|}{DUPRG\textsuperscript{†}} &
			97.1  $\pm$ \small 0.2 &
			83.9  $\pm$ \small 0.5 &
			83.6  $\pm$ \small 0.3 &
			42.0  $\pm$ \small 1.3 &
			\multicolumn{1}{c|}{59.6  $\pm$ \small 0.3} &
			73.2 \\
			\multicolumn{1}{c|}{MIRO\textsuperscript{†}} &
			95.6  $\pm$ \small 0.8 &
			82.2  $\pm$ \small 0.3 &
			82.5  $\pm$ \small 0.1 &
			54.3  $\pm$ \small 0.4 &
			\multicolumn{1}{c|}{54.0  $\pm$ \small 0.3} &
			73.7 \\
			\multicolumn{1}{c|}{GESTUR\textsuperscript{†}} &
			96.0  $\pm$ \small 0.0 &
			82.8  $\pm$ \small 0.1 &
			84.2  $\pm$ \small 0.1 &
			55.7  $\pm$ \small 0.2 &
			\multicolumn{1}{c|}{58.9  $\pm$ \small 0.1} &
			75.5 \\
			\multicolumn{1}{c|}{VL2V} &
			94.3 $\pm$ \small 0.6 &
			82.3 $\pm$ \small 0.3 &
			85.8 $\pm$ \small 0.2 &
			55.3 $\pm$ \small 0.7 &
			\multicolumn{1}{c|}{59.2 $\pm$ \small 0.1} &
			75.5 \\ \midrule
			\multicolumn{7}{c}{ViT-B/16 with parameter efficient methods} \\ \midrule
			\multicolumn{1}{c|}{Adapter\textsuperscript{‡}} &
			92.0  $\pm$ \small 0.5 &
			79.8  $\pm$ \small 0.4 &
			72.9  $\pm$ \small 0.4 &
			44.4  $\pm$ \small 0.8 &
			\multicolumn{1}{c|}{56.2  $\pm$ \small 0.1} &
			69.1 \\
			\multicolumn{1}{c|}{LoRA\textsuperscript{‡}} &
			96.0  $\pm$ \small 0.1 &
			82.7 $\pm$ \small 0.0 &
			83.4  $\pm$ \small 0.1 &
			54.8  $\pm$ \small 0.6 &
			\multicolumn{1}{c|}{58.1  $\pm$ \small 0.1} &
			75.0 \\
			\multicolumn{1}{c|}{VPT\textsuperscript{‡}} &
			96.2  $\pm$ \small 0.3 &
			82.9  $\pm$ \small 0.3 &
			83.4  $\pm$ \small 0.3 &
			54.2  $\pm$ \small 0.7 &
			\multicolumn{1}{c|}{58.9 $\pm$ \small 0.1} &
			75.1 \\
			\multicolumn{1}{c|}{PEGO\textsuperscript{‡}} &
			\textbf{	96.5  $\pm$ \small 0.1} &
			83.2  $\pm$ \small 0.3 &
			83.4  $\pm$ \small 0.3 &
			57.3  $\pm$ \small 0.3 &
			\multicolumn{1}{c|}{\textbf{59.3  $\pm$ \small 0.1}} &
			76.1 \\
			\multicolumn{1}{c|}{JPS $L=12$} &
			95.8  $\pm$ \small 0.1 &
			83.2  $\pm$ \small 0.5 &
			84.4  $\pm$ \small 0.1 &
			57.9  $\pm$ \small 0.6 &
			\multicolumn{1}{c|}{58.8  $\pm$ \small 0.0} &
			76.0 \\
			\multicolumn{1}{c|}{$L=8$} &
			96.0  $\pm$ \small 0.0 &
			\textbf{83.7  $\pm$ \small 0.2} &
			\textbf{84.5  $\pm$ \small 0.1 }&
			\textbf{59.2  $\pm$ \small 0.8} &
			\multicolumn{1}{c|}{58.7 $\pm$ \small 0.1} &
			76.4 \\
			\multicolumn{1}{c|}{$L=4$} &
			96.4  $\pm$ \small 0.1 &
			82.9  $\pm$ \small 0.3 &
			84.0  $\pm$ \small 0.1 &
			54.0  $\pm$ \small 0.2 &
			\multicolumn{1}{c|}{58.0 $\pm$ \small 0.0} &
			75.1 \\
			\multicolumn{1}{c|}{$L=2$} &
			96.3  $\pm$ \small 0.1 &
			82.4  $\pm$ \small 0.0 &
			83.8  $\pm$ \small 0.0 &
			55.1  $\pm$ \small 1.0 &
			\multicolumn{1}{c|}{58.4 $\pm$ \small 0.0} &
			75.2 \\
			\multicolumn{1}{c|}{JPS best} &
			96.4  $\pm$ \small 0.1 &
			\textbf{83.7  $\pm$ \small 0.2} &
			\textbf{84.5  $\pm$ \small 0.1} &
			\textbf{	59.2  $\pm$ \small 0.8} &
			\multicolumn{1}{c|}{58.8  $\pm$ \small 0.0} &
			\textbf{76.6} \\ \hline \bottomrule
	\end{tabular}
	% }
\end{table}

\subsection{Detail Result for Each Dataset\label{dataset_detail}}

\textbf{Results for main experiment.}  Here, we present the generalization accuracy results of our method for each experiment. Given the large dataset size, we report the average results from a single experiment across four datasets. The calculation method is from SWAD \cite{swad} and DomainBed \cite{domainbed}. We have split the results into two tables for clarity, as shown in \cref{dataset1,dataset2}:
\begin{table}[!ht]
	\renewcommand{\arraystretch}{1.25}
	\normalsize
	\centering
	\caption{The result for each dataset (1). }
	\label{dataset1}
	% \resizebox{0.85\textwidth}{!}{
		\begin{tabular}{@{}ccccccccc@{}}
		\toprule
		& \multicolumn{4}{c}{PACS}                       & \multicolumn{4}{c}{VLCS}     \\ \midrule
		\multicolumn{1}{c|}{}             & L=12 & L=8  & L=4  & \multicolumn{1}{c|}{L=2}  & L=12  & L=8   & L=4   & L=2  \\
		\multicolumn{1}{c|}{Experiment 1} & 95.8 & 96.0 & 96.5 & \multicolumn{1}{c|}{96.3} & 84.0  & 83.6  & 82.8  & 82.3 \\
		\multicolumn{1}{c|}{Experiment 2} & 96.1 & 96.0 & 96.5 & \multicolumn{1}{c|}{96.4} & 83.0  & 83.9  & 82.5  & 82.4 \\
		\multicolumn{1}{c|}{Experiment 3} & 95.7 & 95.9 & 96.3 & \multicolumn{1}{c|}{96.2} & 82.7  & 83.4  & 83.3  & 82.4 \\ \midrule
		\multicolumn{1}{c|}{Mean}         & 95.8 & 96.0 & 96.4 & \multicolumn{1}{c|}{96.3} & 83.2  & 83.7  & 82.9  & 82.4 \\
		\multicolumn{1}{c|}{Std.}         & 0.2  & 0.0  & 0.1  & \multicolumn{1}{c|}{0.1}  & 0.5   & 0.2   & 0.3   & 0.0  \\ \midrule
		\multicolumn{1}{c|}{}             & \multicolumn{4}{c|}{OfficeHome}                & \multicolumn{4}{c}{TerraInc} \\ \midrule
		\multicolumn{1}{c|}{}             & L=12 & L=8  & L=4  & \multicolumn{1}{c|}{L=2}  & L=12  & L=8   & L=4   & L=2  \\
		\multicolumn{1}{c|}{Experiment 1} & 84.5 & 84.6 & 84.0 & \multicolumn{1}{c|}{83.9} & 58.7  & 58.7  & 54.3  & 54.2 \\
		\multicolumn{1}{c|}{Experiment 2} & 84.4 & 84.4 & 84.2 & \multicolumn{1}{c|}{83.8} & 57.4  & 58.6  & 53.8  & 56.5 \\
		\multicolumn{1}{c|}{Experiment 3} & 84.3 & 84.5 & 83.9 & \multicolumn{1}{c|}{83.8} & 57.6  & 60.3  & 53.9  & 54.7 \\ \midrule
		\multicolumn{1}{c|}{Mean}         & 84.4 & 84.5 & 84.0 & \multicolumn{1}{c|}{83.8} & 57.9  & 59.2  & 54.0  & 55.1 \\
		\multicolumn{1}{c|}{Std.}         & 0.1  & 0.1  & 0.1  & \multicolumn{1}{c|}{0.0}  & 0.6   & 0.8   & 0.2   & 1.0  \\ \bottomrule
	\end{tabular}
	% }
\end{table}

\begin{table}[!ht]
	\renewcommand{\arraystretch}{1.25}
	\normalsize
	\centering
	\caption{The result for each dataset (2). }
	\label{dataset2}
	% \resizebox{0.5\textwidth}{!}{
		\begin{tabular}{@{}ccccc@{}}
		\toprule
		\multicolumn{5}{c}{DomainNet}                                 \\ \midrule
		\multicolumn{1}{c|}{}             & $L=12$ & $L=8$  & $L=4$  & $L=2$  \\
		\multicolumn{1}{c|}{Experiment 1} & 58.8 & 58.7 & 58.0 & 58.4 \\
		\multicolumn{1}{c|}{Experiment 2} & 58.8 & 58.7 & 58.0 & 58.3 \\
		\multicolumn{1}{c|}{Experiment 3} & 58.8 & 58.6 & 58.1 & 58.4 \\ \midrule
		\multicolumn{1}{c|}{Mean}         & 58.8 & 58.7 & 58.0 & 58.4 \\
		\multicolumn{1}{c|}{Std.}         & 0.0  & 0.1  & 0.0  & 0.0  \\ \bottomrule
	\end{tabular}
	% }
\end{table}

\textbf{Results for accuracy and tunable parameters.} In \cref{accpa}, we compared our approach with various parameter-efficient methods and demonstrated the trade-off between tunable parameters and accuracy. Here, we provide a more detailed explanation: all hyperparameters and JPS results are derived from the main experiments. The key difference is that the PACS results are based on $L=4$, while the other two datasets use $L=8$. Additionally, the results of all parameter-efficient methods are from \cite{PEGO}.
It is worth noting that the parameters selected by JPS can vary depending on gradients and the chosen samples. To provide a comprehensive view, we present the accuracy and the corresponding number of tunable parameters for each experiment. The results are shown in \cref{eff1,eff2,eff3}. Here, Acc. represents the generalization accuracy of the model on the respective dataset, while tunable indicates the number of tunable parameters in the model, excluding the classification linear layer.

\begin{table}[!ht]
		\centering
	\caption{Details in PACS with $L=4$}
	\label{eff1}
	% \resizebox{0.85\textwidth}{!}{
	\begin{tabular}{@{}ccccccccc@{}}
		\toprule
		Experiment             & \multicolumn{2}{c}{Real World}     & \multicolumn{2}{c}{Product}        & \multicolumn{2}{c}{Clipart}        & \multicolumn{2}{c}{Art} \\ \midrule
		\multicolumn{1}{c|}{} & Acc. & \multicolumn{1}{c|}{Tunable} & Acc. & \multicolumn{1}{c|}{Tunable} & Acc. & \multicolumn{1}{c|}{Tunable} & Acc. & Tunable \\ \midrule
		\multicolumn{1}{c|}{1} & 90.390 & \multicolumn{1}{c|}{0.4K} & 89.977 & \multicolumn{1}{c|}{0.5K} & 73.912 & \multicolumn{1}{c|}{0.7K} & 83.213      & 0.7K      \\
		\multicolumn{1}{c|}{2} & 90.304 & \multicolumn{1}{c|}{0.4K} & 90.062 & \multicolumn{1}{c|}{0.5K} & 75.721 & \multicolumn{1}{c|}{0.7K} & 82.286      & 0.5K      \\
		\multicolumn{1}{c|}{3} & 89.816 & \multicolumn{1}{c|}{0.4K} & 90.034 & \multicolumn{1}{c|}{0.5K} & 75.515 & \multicolumn{1}{c|}{0.7K} & 82.853      & 0.5K      \\ \bottomrule
	\end{tabular}
	% }
\end{table}

\begin{table}[!ht]
	\centering
	\caption{Details in VLCS with $L=8$}
	\label{eff2}
	% \resizebox{0.85\textwidth}{!}{
		\begin{tabular}{@{}ccccccccc@{}}
			\toprule
			Experiment             & \multicolumn{2}{c}{V}              & \multicolumn{2}{c}{L}              & \multicolumn{2}{c}{C}              & \multicolumn{2}{c}{S} \\ \midrule
			\multicolumn{1}{c|}{} & Acc. & \multicolumn{1}{c|}{Tunable} & Acc. & \multicolumn{1}{c|}{Tunable} & Acc. & \multicolumn{1}{c|}{Tunable} & Acc. & Tunable \\ \midrule
			\multicolumn{1}{c|}{1} & 85.783 & \multicolumn{1}{c|}{4.0k} & 84.615 & \multicolumn{1}{c|}{4.0k} & 67.718 & \multicolumn{1}{c|}{4.7k} & 96.378     & 6.0k     \\
			\multicolumn{1}{c|}{2} & 85.228 & \multicolumn{1}{c|}{3.8k} & 82.902 & \multicolumn{1}{c|}{3.8k} & 68.188 & \multicolumn{1}{c|}{4.8k} & 97.438     & 6.3k     \\
			\multicolumn{1}{c|}{3} & 84.932 & \multicolumn{1}{c|}{4.0k} & 84.844 & \multicolumn{1}{c|}{3.8k} & 68.424 & \multicolumn{1}{c|}{5.2k} & 97.261     & 6.3k     \\ \bottomrule
	\end{tabular}
	% }
\end{table}

\begin{table}[!ht]
	\centering
	\caption{Details in VLCS with $L=8$}
	\label{eff3}
	% \resizebox{0.85\textwidth}{!}{
		\begin{tabular}{@{}ccccccccc@{}}
			\toprule
			Experiment             & \multicolumn{2}{c}{Real World}     & \multicolumn{2}{c}{Product}        & \multicolumn{2}{c}{Clipart}        & \multicolumn{2}{c}{Art} \\ \midrule
			\multicolumn{1}{c|}{} & Acc. & \multicolumn{1}{c|}{Tunable} & Acc. & \multicolumn{1}{c|}{Tunable} & Acc. & \multicolumn{1}{c|}{Tunable} & Acc. & Tunable \\ \midrule
			\multicolumn{1}{c|}{1} & 90.390 & \multicolumn{1}{c|}{0.4K} & 89.977 & \multicolumn{1}{c|}{0.5K} & 73.912 & \multicolumn{1}{c|}{0.7K} & 83.213      & 0.7K      \\
			\multicolumn{1}{c|}{2} & 90.304 & \multicolumn{1}{c|}{0.4K} & 90.062 & \multicolumn{1}{c|}{0.5K} & 75.721 & \multicolumn{1}{c|}{0.7K} & 82.286      & 0.5K      \\
			\multicolumn{1}{c|}{3} & 89.816 & \multicolumn{1}{c|}{0.4K} & 90.034 & \multicolumn{1}{c|}{0.5K} & 75.515 & \multicolumn{1}{c|}{0.7K} & 82.853      & 0.5K      \\ \bottomrule
	\end{tabular}
	% }
\end{table}
From a more detailed perspective, the number of tunable parameters in our model varies dynamically and fluctuates across different datasets. However, under the same validation set, the tunable parameters in our model remain relatively stable. This result indicates that our model effectively identifies the appropriate tunable parameters. The differences observed across datasets arise due to variations in the data itself. Since these parameters are selected autonomously by the JPS algorithm, our approach exhibits strong adaptability, producing stable results for the same dataset.

% \newpage
% \subsection{Details in Analysis Study\label{ana_detail}}
% We would like to reiterate that all models used in the further analysis are based on \( L=8 \) JPS with all settings the same as the main experiment, primarily due to its superior performance. For most experiments, we have provided detailed numerical results and experimental procedures. The only exception is \textbf{Accuracy for Different $\rho$} where the values are too close to be presented on the graph. Therefore, we provide the detailed numerical values here:

% \begin{table}[!ht]
% 	\renewcommand{\arraystretch}{1.25}
% 	\normalsize
% 	\centering
% 	\caption{The result for Accuracy for Different $\rho$. }
% 	% \resizebox{0.65\textwidth}{!}{
% 		\begin{tabular}{@{}cccccc@{}}
% 		\toprule
% 		$\rho$                            & 0.1    & 0.01   & 0.001  & 0.0001 & 0.00001 \\ \midrule
% 		\multicolumn{1}{c|}{JPS}          & 82.4 & 82.2 & 84.3 & 84.5 & 83.3  \\
% 		\multicolumn{1}{c|}{Direct}       & 81.0 & 80.8 & 83.7 & 84.3 & 84.0  \\
% 		\multicolumn{1}{c|}{w/o variance} & 81.8 & 82.1 & 84.3 & 84.3 & 83.4  \\ \bottomrule
% 	\end{tabular}
% 	% }
% \end{table}

\end{document}